\documentclass{article}
\pdfoutput=1
\usepackage{microtype}
\usepackage{graphicx}
\usepackage{subfigure}
\usepackage{booktabs} %
\usepackage{hyperref}

\usepackage[accepted]{icml2020}

\usepackage[utf8]{inputenc} %
\usepackage[T1]{fontenc}    %
\usepackage{xcolor}

\definecolor{mydarkblue}{rgb}{0,0.08,0.45}

\usepackage{array}
\usepackage{url}            %
\usepackage{amsfonts}       %
\usepackage{nicefrac}       %
\usepackage{microtype}      %

\usepackage{amsmath}
\usepackage{multirow}
\usepackage{color}
\usepackage{bm}
\usepackage{amsthm}

\usepackage{stmaryrd}
\usepackage{grffile}%

\newcommand{\punt}[1]{}

\newtheorem{theorem}{Theorem}
\newtheorem{corollary}{Corollary}

\def\argmax{\mathop{\rm arg\,max}}

\newcommand{\reals}{\mathbb{R}}

\newcommand{\expect}{\mathbb{E}}

\newcommand{\kl}{\textrm{KL}}

\def\argmax{\mathop{\rm arg\,max}}

\newcommand{\bq}{\begin{equation}}
\newcommand{\eq}{\end{equation}}
\newcommand{\ba}{\begin{eqnarray}}
\newcommand{\ea}{\end{eqnarray}}

\newcommand{\mcal}[1]{\mathcal{#1}}
\newcommand{\remove}[1]{}

\newcommand{\ie}{\textit{i}.\textit{e}.}
\newcommand{\eg}{\textit{e}.\textit{g}.}

\newcommand{\tr}{\text{tr}}

\icmltitlerunning{On Implicit Regularization in $\beta$-VAEs}

\begin{document}
\twocolumn[
\icmltitle{On Implicit Regularization in $\beta$-VAEs}

\begin{icmlauthorlist}
\icmlauthor{Abhishek Kumar}{goo}
\icmlauthor{Ben Poole}{goo}
\end{icmlauthorlist}

\icmlaffiliation{goo}{Google Research, Brain Team}

\icmlcorrespondingauthor{Abhishek Kumar}{abhishk@google.com}
\icmlcorrespondingauthor{Ben Poole}{pooleb@google.com}

\icmlkeywords{variational inference, implicit regularization, uniqueness, identifiability, disentanglement}
\vskip 0.3in
]

\printAffiliationsAndNotice{}  %

\begin{abstract}
While the impact of variational inference (VI) on posterior inference in a fixed generative model is well-characterized, its role in regularizing a learned generative model when used in variational autoencoders (VAEs) is poorly understood. 
We study the regularizing effects of 
variational distributions
on learning in generative models from two perspectives. First, we analyze the role that the choice of variational family plays in 
imparting uniqueness to the learned model 
by restricting the set of optimal generative models. Second, we study the regularization effect of the variational family on the local geometry of the decoding model. This analysis uncovers the regularizer implicit in the $\beta$-VAE objective, and leads to an approximation consisting of a deterministic autoencoding objective plus analytic regularizers that depend on the Hessian or Jacobian of the decoding model, unifying VAEs with recent heuristics proposed for training regularized autoencoders. We empirically verify these findings, observing that the proposed deterministic objective exhibits similar behavior to the 
$\beta$-VAE in terms of objective value and sample quality.
\end{abstract}

\section{Introduction}
Variational autoencoders (VAEs) \citep{kingma2013auto, rezende2014stochastic}
and variants such as $\beta$-VAEs \citep{higgins2016beta,alemi2017fixing} have been widely used for density estimation and learning generative samplers of data where latent variables in the model may correspond to factors of variation in the underlying data-generating process \citep{bengio2013representation}.
However, without additional assumptions it may not be possible to recover or identify the correct model \citep{hyvarinen1999nonlinear,locatello2019challenging}.  

We study the regularizing effects of variational distributions used to approximate the posterior on the learned generative model from two perspectives. 
First, we discuss the role of \emph{the choice of variational family} in imparting uniqueness properties to the learned generative model by restricting the set of possible solutions. 
We argue that restricting the variational family in specific ways can rule out non-uniqueness to a corresponding set of transformations on the generative model. 
Next, we study the impact of the variational family on the geometry of the decoding model ($p_\theta(x|z)$) in $\beta$-VAEs.  Assuming that the first two moments exist for the variational family,  
we can relate the structure of the variational distribution's covariance to the Jacobian matrix of the decoding model. For a Gaussian prior and variational distribution, we further use this relation to derive a deterministic objective that closely approximates the $\beta$-VAE objective. %

We perform experiments to validate the correctness of our theory and accuracy of our approximations. On the MNIST dataset, we empirically verify the relationship between variational distribution covariance and the Jacobian of the decoder. In particular, we observe that block-diagonal variational distribution covariance encourages a block diagonal structure on the matrices $J_g(z)^\top J_g(z)$ where $J_g(z)$ is the decoder Jacobian matrix. This generalizes the results from recent work \citep{rolinek2019variational}, which showed %
(for Gaussian observation models) that \emph{diagonal} Gaussian variational distributions encourage orthogonal Jacobians, to general observation models and arbitrary structures on the covariance matrix. We also compare the objective values of the original $\beta$-VAE objective with our derived deterministic objectives. While our deterministic approximation holds for small $\beta$, we empirically find that the two objectives remain in agreement for a reasonably wide range of $\beta$, verifying the accuracy of our regularizers. 
We also train models using a tractable lower bound on the deterministic objectives for CelebA and find that trained models behave similar to $\beta$-VAE in terms of sample quality. %

\section{Latent-variable models and identifiability}
We begin by discussing inherent identifiability issues with  latent-variable models trained using maximum log-likelihood alone.

{\bf Latent-variable models.} We consider latent-variable generative models that consist of a fixed prior $p(z)$, %
and a conditional generative model or {\em decoder}, $p_\theta(x|z)$. In maximum-likelihood estimation, we aim to maximize the marginal log-likelihood of the training data, $\log p_\theta(x) = \log \int dz\, p_\theta(z)p_\theta(x|z)$ in terms of the parameters $\theta$. 
A desirable goal in learning latent-variable models is that the learned latent variables, $z$, correspond simply to ``ground-truth factors of variation'' that generated the dataset \citep{bengio2013representation, higgins2016beta}. For example, in a dataset of shapes we aim to learn latent variables that are in one-to-one correspondence with attributes of the shape like size, color, and position. Learning such ``disentangled'' representations may prove useful for downstream tasks such as interpretability and few-shot learning \citep{gilpin2018explaining, steenkiste2019disentangled}.

{\bf Identifiability and non-uniqueness.} Unfortunately, there is an inherent identifiability issue with latent-variable models trained with MLE: there are many models that have the same marginal data and prior distributions, $p_\theta(x)$ and $p(z)$, but entirely different latent variables $z$.
We can find transformations $r$ of the latent variable $z$ that leave its marginal distribution $p(z)$ the same, and thus the marginal data distribution $p_\theta(x)$ remains the same for a fixed decoder $p_\theta(x|z)$. The transformation $r$ is not necessarily just a simple permutation or nonlinear transformation of individual latents, but {\em mixes} components of the latent representation. %
This leads to a {\em set} of solutions instead of a {\em unique} solution, all of which are equal according to their marginal log-likelihood, but that have distinct representations. Thus identifying the ``right'' model whose latents are in one-to-one correspondence with factors of variation is not feasible with MLE alone.
In our work, we adopt the definition of identifiability from the ICA literature that ignores permutations and nonlinear transformations that act separately on each latent dimension \citep{hyvarinen1999nonlinear}. %
This identifiability issue is well-known in the ICA literature \citep{hyvarinen1999nonlinear} and causal inference literature \citep{peters2017elements}, and has been highlighted more recently in the context of learning disentangled representations with factorized priors \citep{locatello2019challenging}. In Appendix A, %
we summarize the construction from \citet{locatello2019challenging} for factorized priors, and extend their non-uniqueness result to correlated priors.

\section{Uniqueness via variational family}
In the previous section, we highlighted how there can be an infinite family of latent-variable models that all achieve the same marginal likelihood. Here, we show how additional constraints present in the learning algorithm can restrict the space of solutions leading to uniqueness.

We focus on the impact of posterior regularization \citep{zhu2012bayesian, shu2018amortized} via the choice of variational family on the set of solutions to the $\beta$-VAE objective. In variational inference, we optimize not only the generative model, $p(x, z) = p(x|z)p(z)$, but also a variational distribution $q(z|x)$ to approximate the intractable posterior $p(z|x)$. When optimizing over $q(z|x)$, we constrain the distributions to come from some family $\mcal{Q}$. %

We first consider the case when the true posterior lies in the variational family $\mcal{Q}$. Although there are many latent variable generative models with the same marginal data distribution $p(x)$, this equivalence class of models shrinks when we constrain ourselves to models with posteriors $p(z|x) \in \mcal{Q}$. Consider applying a mixing transform $r$, with $|\text{det}\,J_r(z)|=1$, that preserves the prior distribution of $z$ before feeding it to the decoder: $z' = r(z)$. The transformed posterior is given by:
\begin{align}
\begin{split}
p'(z|x)& =\frac{p(x|r(z))p(z)}{p(x)} = \frac{p(x|z')p(r^{-1}(z'))}{p(x)} \\
& = \frac{p(x|z')p(z')}{p(x)} = p(z'|x) = p(r(z)|x),
\end{split}
\end{align}
where we used the facts that $p(z)=p(r(z))=p(r^{-1}(z))$ and $|\text{det}\, J_r(z)|=1$. 
This implies $p(z|x)=p'(r^{-1}(z)|x)$, \ie, the posterior r.v. is transformed by $r^{-1}$ yielding $p'_{z|x}=r^{-1} \# p_{z|x}$ \footnote{For clarity, we denote the transformed distribution using pushforward notation: where $f \# p$ denotes the distribution created by transforming samples from $p$ using the determinstic function $f$. %
}.
If $p_{z|x} \in \mcal{Q}$ but $p'_{z|x} \not\in \mcal{Q}$, then the best ELBO we can achieve in the transformed model $p'$ is strictly less than the untransformed model $p$. Thus even though $r$ leaves the marginal data likelihood unchanged, it yields a worse ELBO and thus can no longer be considered an equivalent model under the training objective. The question remains that for a distribution family $\mcal{Q}$ and a $q\in\mcal{Q}$, what class of transformations $R$ guarantee that $r^{-1} \# q \not \in Q$ for all $r \in R$? This has to be answered on a case-by-case basis, but we consider an example below.

\subsection{Example: Isotropic Gaussian prior and orthogonal transforms}
\label{sec:isoortho}
An isotropic Gaussian prior is invariant under orthonormal transformations (characterized by orthogonal matrices with unit norm columns). 
For a mean-field family of distributions $\mcal{Q}$, the necessary and sufficient conditions for $r^{-1}\# p_{z|x} \in \mcal{Q}$, given $p_{z|x} \in \mcal{Q}$, for all $r$ belonging to the set of orthonormal transformations, are %
\citep{skitovitch1953property,darmois1953analyse,lukacs1954property}: (i) $p_{z|x}$ is a factorized Gaussian, and (ii) variances of the individual latent dimensions of the posterior are all equal. %
If either of these two conditions are not satisfied, the orthogonal symmetry in the generative model, which arises from using the isotropic Gaussian prior, is broken. When all latent variances are unequal, the only orthogonal transforms $r$ allowed on the generative model will be the permutations. This Darmois-Skitovitch characterization of the Gaussian random vectors has also been used to prove identifiability results for linear ICA \citep{hyvarinen1999nonlinear}. However, in our context, we do not need any linearity assumption on the conditional generative model $p(x|z)$.

Probabilistic PCA \citep{tipping1999probabilistic} is one example in this model class with $p(z)=\mcal{N}(0,I),\, p(x|z)=\mcal{N}(Az,\sigma^2 I)$ and $p(x)=\mcal{N}(0, AA^\top +\sigma^2 I)$, that has a Gaussian posterior which is factored when $A$ has orthogonal columns. However, there are equivalent models  $p(x|z)=(AVz,\sigma^2 I)$ with $V$ orthonormal that yield the same $p(x)$ and still have a factorized posterior if the columns of $A$ have equal norms.  This symmetry vanishes only when all columns of $A$ have different norms (condition (ii) mentioned earlier).

\subsection{Uniqueness in $\beta$-VAE}
The choice of variational family can lead to uniqueness even when the true posterior is not in the variational family. Consider %
the $\beta$-VAE objective \citep{higgins2016beta}:
\begin{align}
L(q, p) := \expect_{q(z|x)} \log p_{x|z}(x|z) - \beta\, \kl(q(z|x)\Vert p_z(z)).
\label{eq:elbo_nonpara}
\end{align}
Note that we maximize the objective  w.r.t. only the decoder $p_{x|z}$ (not the prior $p_z$) but use the shorthand $L(q,p)$ in place of $L(q,p_{x|z})$ for brevity.
Let $\mcal{P}$ denote the set of all generative models with fixed prior $p_z$ but arbitrary decoder $p_{x|z}$. Let $R$ be the class of (deterministic) mixing transforms that leave the prior invariant, %
\ie, $p_z = r \# p_z$. Note that %
$r\in R \Rightarrow r^{-1}\in R$. Given a solution $(q^*,p^*)$, s.t., $(q^*,p^*) = \max_{q \in \mcal{Q}, p\in \mcal{P}} L(q, p)$, we would like to know %
if there is a transformational non-uniqueness associated with it, \ie, if it is possible to transform the generative model by $r$ and get the same marginal $p(x)$ as well as same objective value of \eqref{eq:elbo_nonpara}. %
The following result gives a condition under which this non-uniqueness is ruled out. 
\begin{theorem}[Proof in Appendix B] %
Let $R$ be the class of mixing transforms that leave the prior invariant, and $p_r$ be the generative model obtained by transforming the latents of the model $p$ by $r\in R$. %
Let $L(q^*, p^*) = \max_{q \in \mcal{Q}, p \in \mcal{P}} L(q, p)$.
Let $\mcal{\widetilde{Q}}$ be the completion of $\mcal{Q}$ by $R$, \ie, $\mcal{\widetilde{Q}}=\{r\#q: q\in\mcal{Q}, r\in R\}\cup \mcal{Q}$.  If $\mcal{Q}$ is s.t. $q \in \mcal{Q} \Rightarrow r\#q \not\in \mcal{Q}$ for all $r\in R$, and $\argmax_{q\in\mcal{\widetilde{Q}}}L(q,p^*)$ is unique, then $\max_{q \in \mcal{Q}} L(q, p_r^*) < L(q^*, p^*)$.
\label{thm:main}
\end{theorem}

\begin{corollary}[Proof in Appendix B] %
If the set $\mcal{\widetilde{Q}}$ defined in Theorem~\ref{thm:main} is convex, then $\argmax_{q\in\mcal{\widetilde{Q}}}L(q,p^*)$ is unique. When all other assumptions of Theorem~\ref{thm:main} are met, we have $\max_{q \in \mcal{Q}} L(q, p_r^*) < L(q^*, p^*)$.
\end{corollary}

\begin{figure}[t]
\centering
\includegraphics[width=0.4\textwidth]{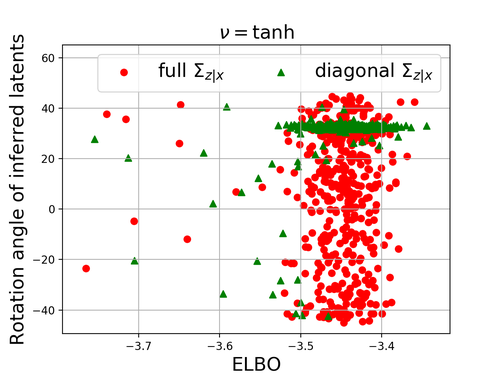}
\caption{Rotation angle between true latents and inferred latents for  VAEs (multiple runs): %
diagonal $\Sigma$ encourages unique solution in terms of inferred latents as the ELBO increases. Please refer to text for details. More experiments are in Appendix C. %
}
\label{fig:ppcaexp}
\end{figure}

Hence, for a choice of variational family $\mcal{Q}$ such that the tuple $(\mcal{Q}, R)$ satisfies the condition of Theorem~\ref{thm:main}, transforming the latents by $r\in R$ will result in decreasing the optimal value of the $\beta$-VAE objective, thus breaking the non-uniqueness. 
From Section \ref{sec:isoortho}, we know that using a Gaussian prior, and restricting $\mathcal{Q}$ to factorized \emph{non-Gaussian} distributions is one such choice that ensures a unique $(q^*, p^*)$ when $R$ is the class of orthogonal transformations. Thus the choice of prior and variational family can help to guarantee uniqueness of models trained with the $\beta$-VAE objective w.r.t. transformations in $R$.
We conduct an experiment highlighting this empirically. 
We generate $4k$ samples in two dimensions using $x = \nu(Wz) + \epsilon$, with $z\sim\mcal{N}(0,I), \epsilon\sim\mcal{N}(0,.05^2I)$. Entries of $W\in\reals^{2\times 2}$ are sampled from $\mcal{N}(0,1)$ and normalized such that columns have different $\ell_2$ norms (2 and 1, respectively). $\nu$ is either Identity (Probabilistic PCA) or a nonlinearity (tanh, sigmoid, elu). A VAE is trained on this data, with decoder of the form $x = \nu(Az)+b$ and variational posterior of the form $q(z|x)=\mcal{N}(Cx+d,\Sigma)$ with $A,b,C,d$ and $\Sigma$ learned. We do multiple random runs of this experiment and measure the rotation angle between true latents and inferred posterior means (after ignoring permutation and reflection of axes so the angle ranges in $(-45,45]$).
These runs are trained with a varying number of iterations so they reach solutions with varying degrees of suboptimality (ELBO).  Fig.~\ref{fig:ppcaexp} shows a scatter plot of rotation angle vs ELBO for $\nu$=tanh. When $\Sigma$ is diagonal (mean-field $q(z|x)$), there is a unique solution in terms of inferred latents as the ELBO increases (ignoring the permutations and reflections of latent axes). In Appendix C,  %
we show this uniqueness is not due to additional inductive biases in the optimization algorithm, and provide results for unamortized VI and other activations $\nu$.

{\bf Disentanglement.~}
\citet{locatello2019challenging, mathieu2019disentangling} highlight that learning disentangled representations using only unlabeled data is impossible without additional inductive biases due to this non-uniqueness issue of multiple generative models resulting in same $p(x)$. 
While \citet{mathieu2019disentangling} detail how the $\beta$-VAE objective is invariant to rotation of the latents with a flexible %
variational distribution, 
they do not consider the role mean-field plays in limiting the set of solutions.
Based on the discussion in this section, it can be seen that the choice of variational family can be one such inductive bias, \ie, it can be chosen in tandem with the prior to rule out issues of non-uniqueness due to certain transforms. Most existing works on unsupervised disentanglement using $\beta$-VAE's \citep{higgins2016beta,kumar2017variational,chen2018isolating,kim2018disentangling,burgess2018understanding}
use a standard normal prior and factored Gaussian variational family which should not have this orthogonal non-uniqueness issue (\ie, the construction in the proof of Theorem 1 of \citep{locatello2019challenging} will change the ELBO), unless two or more variational %
distribution covariances are equal. 
Note that this observation is limited to the theoretical uniqueness aspect of the models and does not contradict the empirical findings of \citet{locatello2019challenging} who observe that $\beta$-VAEs converge to solutions with different disentanglement scores depending on the random seeds. There could be multiple reasons for this, such as models differing in their training objective values (\ie, suboptimal solutions where our theoretical results do not apply), disentanglement metrics being sensitive to non-mixing transforms or the finite dataset size. We leave further investigation of these empirical aspects %
to future work.

\section{Deterministic approximations of $\beta$-VAE}
\label{sec:detapprox}
Next, we analyze how the covariance structure of the variational family can regularize the geometry of the learned generative model.
We start with the $\beta$-VAE objective of Eq.~\eqref{eq:elbo_nonpara}, which reduces to the ELBO for $\beta=1$. 
We use $p_\theta(x|z)$ to denote the decoding model parameterized by $\theta$, and $q_\phi(z|x)$ to denote  the amortized inference model (encoding model) parameterized by $\phi$. %
We will sometimes omit the parameter subscripts for brevity. We explicitly consider the case where %
the first two moments exist for the  variational %
distribution $q_\phi(z|x)$. 
We denote the latent dimensionality by $d$ and the observation dimensionality by $D$.

Let us denote $f_x(z)=\log p(x|z)$ and view $\log p(x|z)$ as a scalar function of $z$ for a given $x$. The second-order Taylor series expansion of $f_x(z)$ around the first moment (mean)  $\mu_{z|x} = \expect_{q_{\phi(\cdot|x)}}\left[z\right] = h_\phi(x)$ %
is given by:
\begin{align}
\begin{split}
f_x(z) \approx & \log p(x|h_\phi(x)) +  J_{f_x}(h_\phi(x))(z-h_\phi(x))\\  
& + \frac{1}{2}(z-h_\phi(x))^\top H_{f_x}(h_\phi(x))(z-h_\phi(x)), \end{split} 
\label{eq:taylor}
\end{align}
\noindent where $J_{f_x}(z)\in \reals^{1\times d}$ and $H_{f_x}(z)\in\reals^{d\times d}$ are the Jacobian and Hessian of $f_x$ evaluated at $z$. Note that the expectation of the second term under $q_{\phi(z|x)}$ is zero as $\expect_{q_{\phi(\cdot|x)}}\left[z\right] = h_\phi(x)$. Substituting the approximation \eqref{eq:taylor} in the original $\beta$-VAE objective  \eqref{eq:elbo_nonpara}, we obtain the following approximation 
\begin{align}
\begin{split}
L(p, q) \approx \log p_\theta(x|h_\phi(x)) & + \frac{1}{2}\tr( H_{f_x}(h_\phi(x))\Sigma_{z|x}) \\
& - \beta\, \kl(q_\phi(z|x)\Vert p(z)),
\end{split}
\label{eq:elbo_approx}
\end{align}
\noindent where $\Sigma_{z|x}=\expect_{q(z|x)}[(z-\mu_{z|x})(z-\mu_{z|x})^\top]$ is the %
covariance of the variational distribution. 
For this approximation to be accurate, we either need the higher central moments of the variational distribution to be small or the higher-order derivatives  $\nabla^n_z \log p(x|z)|_{z=h_\phi(x)}$ ($n\geq 3$) to be small.
For small values of $\beta$, the variational distribution %
will be  fairly concentrated around the mean, resulting in small higher central moments (if any). While this objective no longer requires sampling, it involves a Hessian of the decoder, which can be computationally expensive to evaluate. Next we show how this Hessian can be computed tractably given certain network architectures.

Let $g_\theta:Z\to \reals^D$ denote the determinstic component of the decoder that maps from a latent vector $z$ to the parameters of the observation model (\eg, to means of the Gaussian pixel observations, to probabilities of the Bernoulli distribution for binary images, etc.). %
We can further expand the Hessian $H_{f_x}(z)$ to decompose it as the sum of  two terms, one involving the Jacobian of the decoder mapping $J_g(z)\in\reals^{D\times d}$, and the other involving the Hessian of the decoder mapping, as follows:
\begin{align}
\begin{split}
H_{f_x}(z) =\, & \nabla_z^2 \log p(x|z) = \nabla_z^2 \log p(x; g(z)) \\
 =\, & \nabla_z [  (\nabla_{g(z)}\log p(x;g(z)))J_g(z)] \\
=\, & J_g(z)^\top (\nabla_{g(z)}^2 \log p(x;g(z))) J_g(z) \,+ \\ 
 &(\nabla_{g(z)}\log p(x;g(z)))\nabla_z^2 g(z)
\end{split}
\label{eq:hess_decomp}
\end{align}
It can be noted that $\nabla_z^2 g(z)$ in the second term is almost always zero when the decoder network uses only  piecewise-linear activations (\eg, relu, leaky-relu), which is the case in many implementations. Denoting $H_{p_x}(g(z))=\nabla_{g(z)}^2 \log p(x;g(z))$, this leaves us with 
\begin{align}
H_{f_x}(z) = J_g(z)^\top H_{p_x}(g(z)) J_g(z)
\label{eq:hess_approx}
\end{align}
In the case of other nonlinearities, we can make the approximation that $\nabla_z^2 g(z)$ is small and can be neglected. Substituting the Hessian $H_{f_x}(z)$ from Eq.~\eqref{eq:hess_approx} into the approximate ELBO \eqref{eq:elbo_approx}, we obtain
\begin{align}
\begin{split}
\max_{g,h} \,\,& \log p(x|h(x)) - \beta\, \kl(q_\phi(z|x)\Vert p(z))\, + \\ &\frac{1}{2}\tr(J_g(h(x))^\top H_{p_x}(g(h(x))) J_g(h(x)) \Sigma_{z|x})
\end{split} 
\label{eq:elbo_approx2}
\end{align}

The above deductions in \eqref{eq:elbo_approx} and
\eqref{eq:elbo_approx2}, obtained using the second order approximation of $\log p(x|z)$, imply that stochasticity of the variational distribution %
translates to regularizers that encourage the \emph{alignment} of $H_{f_x}(h(x))$ and $J_g(h(x))^\top H_{p_x}(g(h(x))) J_g(h(x))$ with %
covariance $\Sigma_{z|x}$ by maximizing their inner product with it. Note that for observation models that independently model stochasticity at individual pixels (\eg, pixel-wise independent Gaussian, Bernoulli, etc.), the matrix $H_{p_x}(g(h(x)))$ is diagonal and thus efficient to compute. In particular, for independent standard Gaussian observations (\ie, $p(x;g(z))=\mcal{N}(g(z),I)$), $H_{p_x}(g(h(x)))=-I$, the identity matrix of size $D$. 

\section{Impact of covariance structure of $q$}
The determinstic approximation to the $\beta$-VAE objective allows us to study the relationship between the decoder Jacobian and %
the covariance of the variational distribution. Lets consider the specific case of a standard Gaussian prior, $p(z)=\mcal{N}(0,I)$, and Gaussian variational family, $q_\phi(z|x)=\mcal{N}(h_\phi(x),\Sigma_{z|x})$. The objective in \eqref{eq:elbo_approx} reduces to
\begin{align}
\begin{split}
\log p_\theta(x|h_\phi(x)) + \frac{1}{2}\tr( H_{f_x}(h_\phi(x))\Sigma_{z|x})\, - \\ 
\frac{\beta}{2} \left(\lVert h_\phi(x)\rVert^2 + \tr(\Sigma_{z|x}) - \log |\Sigma_{z|x}| - d\right),
\end{split}
\label{eq:elboapprox_gauss}
\end{align}
\noindent where $|\Sigma_{z|x}|$ denotes the absolute value of the determinant of $\Sigma_{z|x}$. We consider the case when the mean of the variational distribution %
is parameterized by a neural network but the covariance is not amortized. The objective \eqref{eq:elboapprox_gauss} is concave in $\Sigma_{z|x}$, and maximizing it w.r.t. $\Sigma_{z|x}$ gives
\begin{align}
\begin{split}
 & \frac{1}{2}H_{f_x}(h_\phi(x)) - \frac{\beta}{2} \left(I - \Sigma_{z|x}^{-1}\right) = 0 \\
\Longrightarrow & \quad \Sigma_{z|x} = \left(I - \frac{1}{\beta}H_{f_x}(h_\phi(x)) \right)^{-1}
\end{split}
\label{eq:covhess}
\end{align}
Substituting Eq. \eqref{eq:covhess} into \eqref{eq:elboapprox_gauss}, we get the maximization objective of
\begin{align}
\begin{split}
\max \,\,  \log p_\theta(x|h_\phi(x)) & - \frac{\beta}{2} \lVert h_\phi(x)\rVert^2 \, - \\
& \frac{\beta}{2} \log \left|I- \frac{1}{\beta} H_{f_x}(h_\phi(x))\right|.
\end{split}
\end{align}
Following similar steps for the objective \eqref{eq:elbo_approx2} gives the optimal %
covariance of
\begin{align}
\begin{split}
  \Sigma_{z|x} = \left(I - \frac{1}{\beta}J_g(h(x))^\top H_{p_x}(g(h(x))) J_g(h(x)) \right)^{-1},
\end{split}
\label{eq:covjacob}
\end{align}
\noindent and the maximization objective of 
\begin{align}
\begin{split}
&\hspace{-3mm}\log p_\theta(x|h_\phi(x))  - \frac{\beta}{2} \lVert h_\phi(x)\rVert^2 \, - \\
&\hspace{-3mm}\frac{\beta}{2} \log \left|I- \frac{1}{\beta} J_g(h_\phi(x))^\top H_{p_x}(g_\theta(h_\phi(x))) J_g(h_\phi(x))\right|
\end{split}
\label{eq:elbo_gauss_optcov}
\end{align}
This analysis yields an %
interpretation of the $\beta$-VAE objective as a deterministic reconstruction term ($p_\theta(x|h_\phi(x))$, a regularizer on the encodings ($\lVert h_\phi(x)\rVert^2$), and a regularizer involving the Jacobian of the decoder mapping.

\subsection{Regularization effect of the covariance}
The relation between the optimal $\Sigma_{z|x}$ and the Jacobian of the decoder $J_g(z)$ in Eq.~\eqref{eq:covjacob} uncovers how the choice of the structure of $\Sigma_{z|x}$ influences the Jacobian of the decoder. For example, having a diagonal structure on the covariance matrix (mean-field variational distribution) %
should drive the matrix $(J_g(h(x))^\top H_{p_x}(g(h(x))) J_g(h(x)))$ towards a diagonal matrix. As mentioned earlier, for observation models that independently model the individual pixels given the decoding distribution parameters $g(z)$, the matrix $H_{p_x}(g(z))$ is diagonal which implies that the matrix $|H_{p_x}(g(h(x)))|_{\text{abs}}^{1/2}J_g(h(x))$ is encouraged to have orthogonal columns, where $|H_{p_x}(g(h(x)))|_{\text{abs}}$ denotes taking elementwise absolute values of the matrix $H_{p_x}(g(h(x)))$. %

For standard Gaussian observations, we have $H_{p_x}(g(z))=-I$ for all $z$, implying that the decoder Jacobian $J_g(h(x))$ will be driven towards having orthogonal columns. 
For Bernoulli distributed pixel observations, the $i$'th diagonal element of matrix $H_{p_x}(g(h(x)))$ is given by $-1/(1-x_i-[g(h(x))]_i)^2$, where $x_i\in\{0,1\}$ is the $i$'th pixel and $[g(h(x))]_i$ is the predicted probability of $i$'th pixel being $1$. When the decoding network has the capacity to fit the training data well, \ie, $(1-x_i-[g(h(x))]_i)^2$ is close to $1$ for all $i$, the Jacobian matrix $J_g(h(x))$ will still be driven towards having orthogonal columns. %
Orthogonality of the Jacobian has been linked with semantic disentanglement and regularizers to encourage orthogonality have been proposed to that end \cite{ramesh2019spectral}. Our result shows that diagonal covariance of the variational distribution naturally encourages orthogonal Jacobians. 
As is evident from Eq.~\eqref{eq:covjacob}, more interesting structures on the decoder Jacobian matrix can be encouraged by placing appropriate sparsity patterns on $\Sigma_{z|x}^{-1}$. 
Often it may be expensive to encode constraints on the decoder Jacobian directly, as that would typically require computing the Jacobian to apply the constraint. Instead, this relationship highlights that we can impose soft constraints on the decoder Jacobian by restricting the covariance structure for the variational family.

\subsection{Connection with Riemannian metric}
For Gaussian observations, Eq. \eqref{eq:covjacob} reduces to  
$\Sigma_{z|x} = \left(I + \frac{1}{\beta}J_g(h(x))^\top  J_g(h(x)) \right)^{-1}$. The decoder,  $g:Z\to \reals^D$, outputs the mean of the distribution and its image $M=g(Z)\subset \reals^D$
is an embedded differentiable manifold\footnote{This assumes that the decoder uses differentiable nonlinearities which does not hold for piecewise-linear activations such as relu, leaky-relu. However, in principle, they can always be approximated arbitrarily closely by a smooth function for the sake of this argument.} of dimension $d$, when $J_g(z)$ is full rank for all $z$. 
If we inherit the Euclidean geometric structure from $\reals^D$, restricting the inner product to the tangent spaces $T_x M$ gives us a Riemannian metric on $M$. The pullback of this metric to $Z$ gives the corresponding Riemannian metric on $Z$, which is given by a symmetric positive-definite matrix field $G(z)=J_g(z)^\top J_g(z)$. Recent work has used this metric to study the geometry of deep generative models \citep{shao2018riemannian,chen2017metrics,arvanitidis2017latent,kuhnel2018latent}. Given the relation in Eq. \eqref{eq:covjacob}, we can think of variational distribution covariance %
as indirectly inducing a metric on the latent space for Gaussian observations as $G(z)=J_g(z)^\top J_g(z)=\beta(\Sigma_{z|x}^{-1}-I)$ for $z=h(x)$. It will be interesting to investigate the metric properties of $\Sigma_{z|x}^{-1}$ in the context of deep generative models, \eg, for geodesic traversals, curvature calculations, etc. \citep{shao2018riemannian,chen2017metrics}. 
It can be noted that training fixed covariance VAEs,\ie, with $\Sigma_{z|x}$ being same for all $x$, encourages learning of flat manifolds. Interestingly, learning flat manifolds has been explicitly encouraged in StyleGAN2 \citep{karras2020analyzing} and in recent work on VAEs \citep{kato2020rate,chen2020learning}, particularly with orthogonal Jacobian matrix, which will be implicitly encouraged in VAEs if we fix $\Sigma_{z|x}$ to be diagonal and same for all $x$. 

\subsection{Training deterministic approximations of $\beta$-VAE}
Beyond the theoretical understanding of the implicit regularization in $\beta$-VAE, it should also be possible to use the proposed objectives for training the models, particularly when the observation dimensions are conditionally independent as the Hessian $H_{p_x}(\cdot)$ is diagonal. This will also help in investigating the accuracy and utility of our deterministic approximations to the $\beta$-VAE objective.
We consider the case of Gaussian observation model for simplicity as the matrix $H_{p_x}(g(z))=-I$, reducing the objective to: 
\begin{align}
\begin{split}
\min_{g,h} \,\, \frac{1}{2}\lVert x- & g(h(x)) \rVert^2  + \frac{\beta}{2} \lVert h(x)\rVert^2 \, + \\
& \frac{\beta}{2} \log \left|I+ \frac{1}{\beta} J_g(h(x))^\top  J_g(h(x))\right|.
\end{split}
\label{eq:elbo_gauss_obs}
\end{align}
We refer to the objective \eqref{eq:elbo_gauss_obs} as Gaussian Regularized AutoEncoder ({\bf GRAE}).
However, for large latent dimensions, computing full Jacobian matrices can substantially increase the training cost. 
To reduce this cost, we propose training with a stochastic approximation of a upper bound of the regularizers in objectives \eqref{eq:elbo_gauss_optcov} and \eqref{eq:elbo_gauss_obs}.  We use Hadamard inequality for determinants of positive definite matrices $A$, \ie, $\text{det}(A)\leq \prod_i A_{ii}$ with equality for diagonal $A$, and approximate the  regularizer of \eqref{eq:elbo_gauss_obs} as:
\begin{align}
\begin{split}
\log & \left|I+ \frac{1}{\beta} J_g(h(x))^\top  J_g(h(x))\right| \\ 
& \leq \sum_i \log \left(1+\frac{1}{\beta} \big\lVert[J_g(h(x))]_{:i}\big\rVert_2^2\right) 
\end{split}
\label{eq:ubreg}
\end{align}
Let $p_c$ be a discrete distribution on the column indices of the Jacobian matrix (\eg, uniform distribution). We can approximate the summation by sampling $k$ column indices $\{c_i\}_1^k$ from the distribution $p_c$ and using $\frac{1}{k}\sum_{i=1}^k \frac{1}{p_c(c_i)}\log \left(1+\frac{1}{\beta} \big\lVert[J_g(h(x))]_{:c_i}\big\rVert_2^2\right)$. Lacking any more information, we simply use a uniform distribution for $p_c$, and $k=1$ in all our experiments. We refer to the objective with this more tractable regularizer as {\bf GRAE}$_{\bm{\approx}}$. 
Similar steps can be followed for the regularizer in \eqref{eq:elbo_gauss_optcov}.
The upper bound in  \eqref{eq:ubreg} should be tight for mean-field variational distributions %
by virtue of the relation in Eq.~\eqref{eq:covjacob}.
For such cases (\ie, sparse %
precision matrices for variational distributions, including, diagonal), we can impose an additional orthogonal penalty on the corresponding Jacobian columns (in line with Eq. \eqref{eq:covjacob}) by penalizing their normalized dot products. This can also be implemented in stochastic form by sampling a pair of Jacobian columns. %

It is possible to relate the regularizer in \eqref{eq:elbo_gauss_obs}
with earlier work on regularized autoencoders. 
We can use the Taylor series for $\log (1+x)$ around $x=0$ to approximate the regularizer in \eqref{eq:elbo_gauss_obs}, which converges when the singular values of $J_g(h(x))$ are smaller than $\sqrt{\beta}$. 
Restricting the Taylor series to just first order term yields the simplified regularizer of $\frac{1}{2}\lVert J_g(h(x))\rVert_F^2$ (note that this approximation is crude:  $\beta$ cancels out and does not influence the decoder regularizer in any way). 
Similar approximations can be obtained for \eqref{eq:elbo_gauss_optcov} which will yield the regularizer of $\lVert J_g(h(x))[-H_{p_x}(g(h(x)))]^{1/2}\rVert_F^2$. This simplified regularizer of $\frac{1}{2}\lVert J_g(h(x))\rVert_F^2$ weighted by a free hyperparameter has been used in a few earlier works with autoencoders, including recently in \citep{ghosh2019variational}, where it was motivated in a rather heuristic manner for encouraging smoothness. On the other hand, we are specifically motivated by uncovering the regularization implicit in $\beta$-VAE and the regularizer in \eqref{eq:elbo_gauss_obs}  emerges as part of this analysis. 

\begin{figure*}[t]
\centering
\subfigure[$\beta=0.2,d=8,b=2$]{\label{fig:blksz2_a}\includegraphics[width=0.11\textwidth]{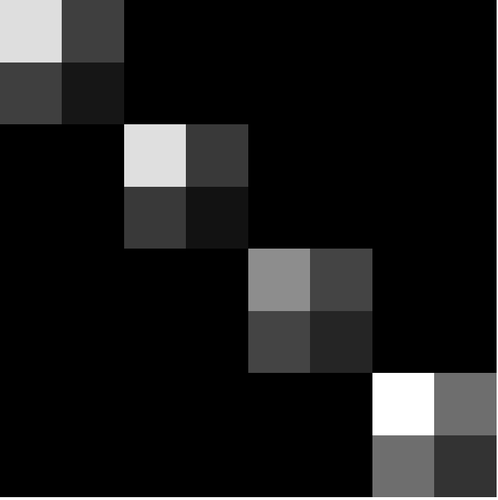} 
\includegraphics[width=0.11\textwidth]{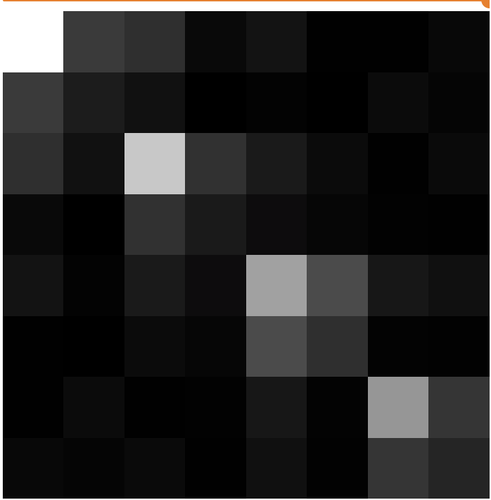}}~~~~~
\subfigure[$\beta=0.4,d=8,b=2$]{\label{fig:blksz2_b}\includegraphics[width=0.11\textwidth]{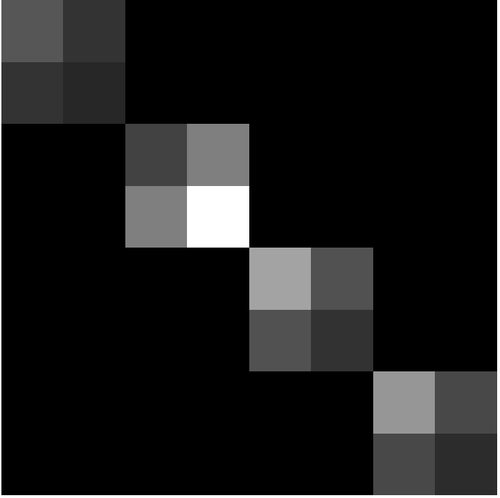} 
\includegraphics[width=0.11\textwidth]{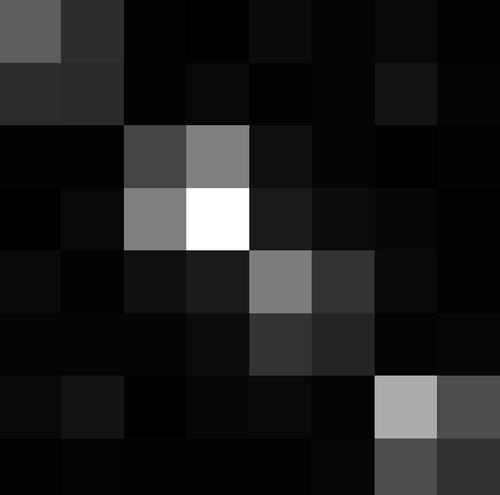}}~~~~~
\subfigure[$\beta=0.6,d=8,b=3$]{\label{fig:blksz3_b}\includegraphics[width=0.11\textwidth]{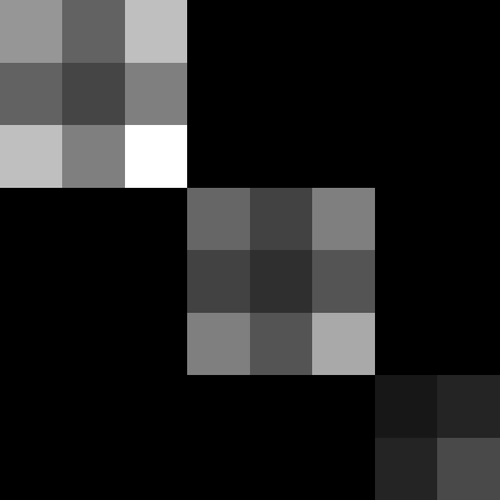}
\includegraphics[width=0.11\textwidth]{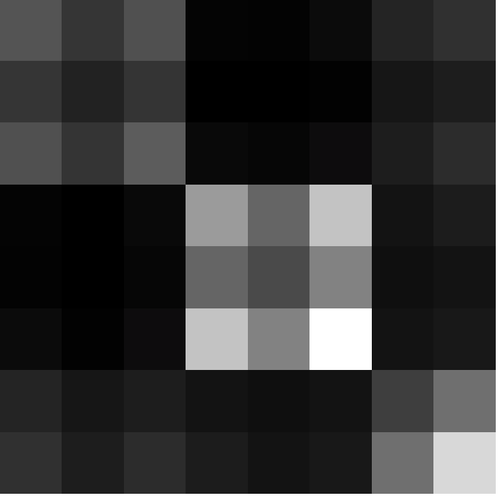}}~~~~~~
\subfigure[$\beta=0.8,d=20,b=3$]{\label{fig:blksz3_c}\includegraphics[width=0.11\textwidth]{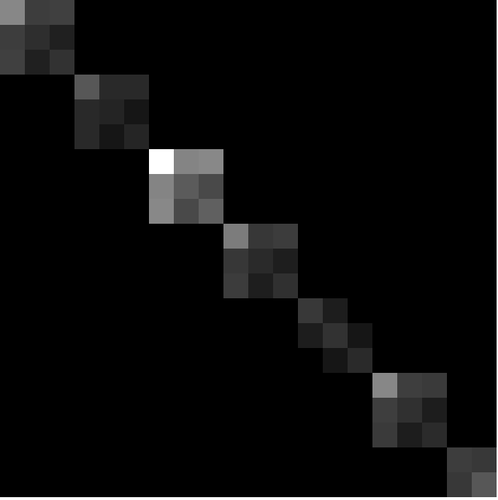}
\includegraphics[width=0.11\textwidth]{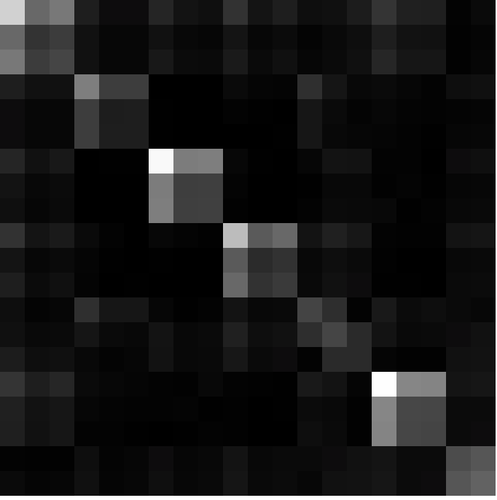}}
\caption{MNIST: Precision matrices $\Sigma_{z|x}^{-1}$ of the variational distributions (left) and the corresponding $J_g(h(x))^\top J_g(h(x))$ (right), for different values of $\beta$ and latent dimensionality $d$. Block size $b$ for the covariance was taken to be 2 or 3. More plots are shown in Appendix E.} %
\label{fig:blksz23_mnist}
\end{figure*}

\begin{figure*}[!htbp]
\centering
\hspace{-4mm}\includegraphics[width=0.27\textwidth]{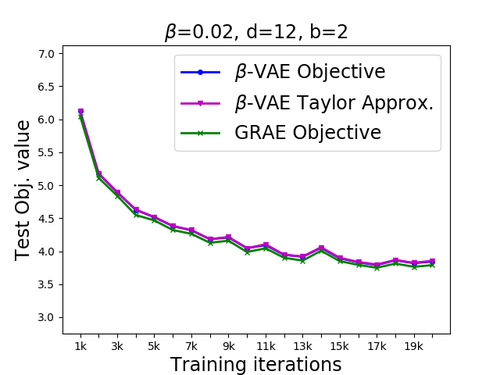}\hspace{-4mm}\includegraphics[width=0.27\textwidth]{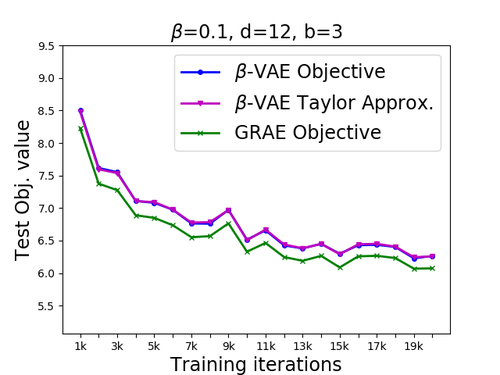}\hspace{-4mm}\includegraphics[width=0.27\textwidth]{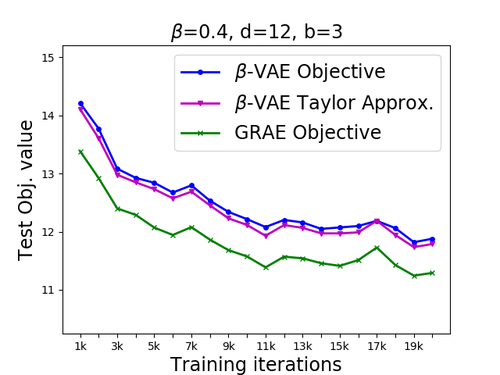}\hspace{-6mm}\includegraphics[width=0.27\textwidth]{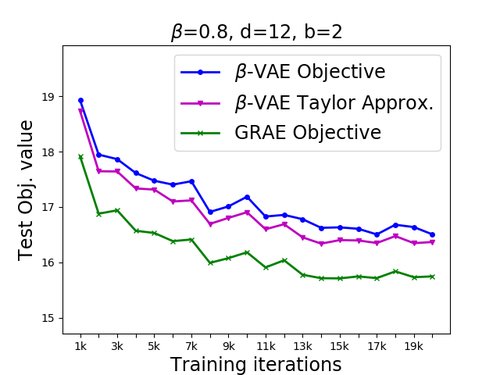}
\caption{MNIST: Comparison of objective values of $\beta$-VAE, its Taylor approximation (Eq. \eqref{eq:elbo_approx2}), and GRAE (Eq. \eqref{eq:elbo_gauss_obs}) for different values of $\beta$, %
and different values of block size $b$ for %
covariance matrix of variational distributions. Latent dimensionality $d$ is fixed to 12. More plots can be viewed in Appendix F.} %
\label{fig:compare_obj_mnist}
\end{figure*}

\vspace{-1mm}   
\section{Related Work}
\label{sec:related}
Our work resembles existing work on marginalizing out noise to obtain deterministic regularizers. \citet{maaten2013learning} proposed deterministic regularizers obtained with marginalizing feature noise belonging to exponential family distributions. \citet{chen2012marginalized} marginalize noise in denoising autoencoders to learn robust representations. Marginalization of dropout noise has also been explored in the context of linear models as well as deep neural networks \cite{srivastava2013improving,wang2013fast,srivastava2014dropout, poole2014analyzing}. Recently, \citet{ghosh2019variational} argued for replacing stochasticity in VAEs with deterministic regularizers resulting in deterministic regularized autoencoder objectives. However, the regularizers considered there were motivated from a heuristic perspective of encouraging smoothness in the decoding model. Another recent work \citep{kumar2020learning} uses injective probability flow to derive  autoencoding objectives with Jacobian-based regularizers. Different from these works, our goal in Sec. \ref{sec:detapprox} is to uncover the regularizers implicit in the $\beta$-VAE objective by marginalizing out noise coming from the variational distribution, and solving for its optimal covariance, yielding the regularizers in \eqref{eq:elbo_approx}, \eqref{eq:elbo_approx2},  \eqref{eq:elbo_gauss_optcov} and \eqref{eq:elbo_gauss_obs}.%

Recent work by \citet{rolinek2019variational} considered the case of VAEs with a Gaussian prior, Gaussian variational distribution with diagonal covariance and Gaussian observations, and showed that diagonal posterior covariance encourages the decoder Jacobian to have orthogonal columns. In our work, we explicitly characterize the regularization effect of variational distribution covariance, obtaining insights into the effect of arbitrary sparsity structure of the variational distribution's precision matrix $\Sigma_{z|x}^{-1}$ on the decoder Jacobian (Eq.~\eqref{eq:covjacob}). This leads to regularized autoencoder objectives with specific regularizers (\eg, Eq. \eqref{eq:elbo_gauss_optcov}). Our analysis also generalizes to non-Gaussian observation models.
\citet{park2019variational} address the amortization gap and limited expressivity of diagonal variational posteriors in VAEs 
by using Laplace approximation around the \emph{mode of 
the posterior}. On the other hand, our analysis starts with the Taylor approximation of the conditional $p(x|z)$ around the \emph{mean of the variational posterior} and then further diverges from \cite{park2019variational} towards different goals.
\citet{stuhmer2019independent} recently studied structured non-Gaussian \emph{priors} for unsupervised disentanglement, motivated by rotational non-uniqueness issues associated with isotropic Gaussian prior, however there will still be mixing transforms that leave these priors invariant. 
\citet{burgess2018understanding} discuss how disentangled representations may emerge when using a mean field variational family, but do not present a rigorous argument connected to uniqueness. Recently, \citet{khemakhem2019variational} provided uniqueness results for conditional VAEs with \emph{conditionally} factorized priors.

\section{Experiments}
\begin{table*}[t]
\caption{CelebA samples using the standard normal prior for $\beta$-VAE and GRAE$_\approx$ (Eq.~\eqref{eq:elbo_gauss_obs} with the regularizer approximated by stochastic approximation of the upper bound of Eq.~\eqref{eq:ubreg}) for different values of $\beta$. Samples from models trained using the GRAE$_\approx$ objective have similar smoothness/blurriness as samples from $\beta$-VAE models for a wide range of $\beta$. More samples can be viewed in  Appendix G.} %
\centering
\begin{tabular}{lccccc}
& $\beta=0.02$ & $\beta=0.06$ &  $\beta=0.1$ &  $\beta=0.4$ & $\beta=0.8$ \\
\rotatebox[origin=c]{90}{~~~~~~~~~~~~~~~~~~~~~~~~~~{\bf VAE} Samples}\hspace{-4mm} & \includegraphics[width=0.18\textwidth]{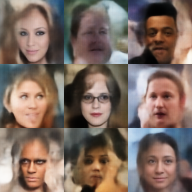}\hspace{-2mm} &
\includegraphics[width=0.18\textwidth]{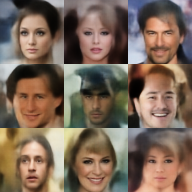}\hspace{-2mm} &
\includegraphics[width=0.18\textwidth]{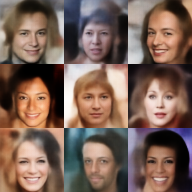}\hspace{-2mm} &
\includegraphics[width=0.18\textwidth]{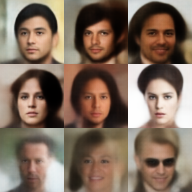}\hspace{-2mm} &
\includegraphics[width=0.18\textwidth]{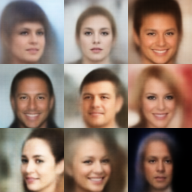} \vspace{-20mm}\\
\rotatebox[origin=c]{90}{~~~~~~~~~~~~~~~~~~~~~~~~{\bf GRAE$_{\bm{\approx}}$ Samples}} \hspace{-4mm} &
\includegraphics[width=0.18\textwidth]{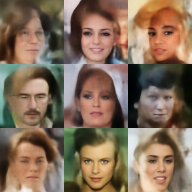}\hspace{-2mm} &
\includegraphics[width=0.18\textwidth]{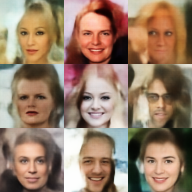}\hspace{-2mm} &
\includegraphics[width=0.18\textwidth]{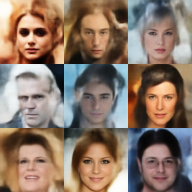}\hspace{-2mm} &
\includegraphics[width=0.18\textwidth]{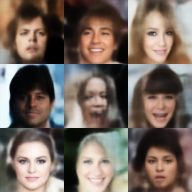}\hspace{-2mm} &
\includegraphics[width=0.18\textwidth]{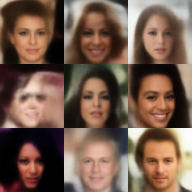} \vspace{-20mm}
\end{tabular}
\label{tab:celebasamples_betasweep}
\end{table*}

\begin{figure}[t]
\centering
\includegraphics[width=0.35\textwidth]{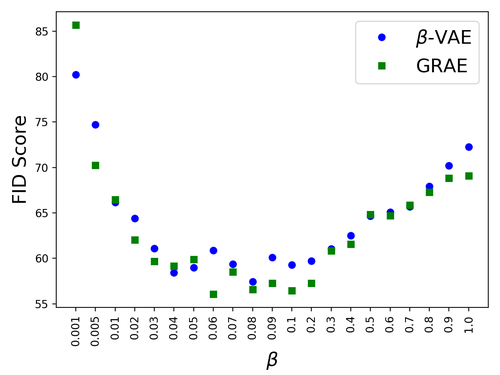}
\caption{FID scores for CelebA generated samples using $\beta$-VAE and GRAE$_\approx$: averaged over 5 runs for each $\beta$. All standard deviations are less than $1$. Samples from the two models stay close in terms of FID scores, further substantiating the validity of our derived deterministic approximation.}
\label{fig:fidcomp}
\end{figure}

We conduct experiments on MNIST \citep{Lecun98gradient-basedlearning} and CelebA  \citep{liu2015deep} to test the closeness of the proposed deterministic regularized objectives to the original $\beta$-VAE objective and the relation between the %
covariance of the variational distribution and the Jacobian of the decoder. Standard train-test splits are used in all experiments for both datasets. We use Gaussian distributions for the prior, variational distribution, and observation model in all our experiments for both $\beta$-VAE and the proposed deterministic objectives (which reduces to the GRAE objective of Eq.~\eqref{eq:elbo_gauss_obs} for the Gaussian case). Note that for the isotropic Gaussian observation model we use, multiplication by $\beta$ for the KL term in $\beta$-VAE is equivalent to assuming an isotropic decoding covariance of $\frac{1}{\beta}I$ in a VAE.
We use $64\times 64$ cropped images for CelebA faces as in several earlier works. 
We use 5 layer CNN architectures for both the encoder and decoder, %
with elu activations \citep{clevert2015fast} in all hidden layers. 
We use elu to test the approximation quality of GRAE objective when Eq.~\eqref{eq:hess_approx} is not exact. For piecewise-linear activations such as relu, Eq.~\eqref{eq:hess_approx} is exact and the approximation is better compared to elu.
The same architecture is used for both $\beta$-VAE and GRAE, except an additional fully-connected (fc) output layer in the encoder for $\beta$-VAE that produces the standard deviations of the amortized mean field variational distributions. 
For non-factorized Gaussian variational distribution, we mainly work with block-diagonal covariance matrices, which restricts the corresponding precision matrices to also be block-diagonal. This is implemented by an additional fc layer in the encoder outputting the entries of the appropriate Cholesky factors $C_{z|x}$ such that $\Sigma_{z|x}=C_{z|x}C_{z|x}^\top$ is block diagonal with desired block sizes. 
All MNIST and CelebA models are trained for 20k and 50k iterations respectively, using the Adam optimizer \citep{kingma2014adam}.  

{\bf Relation between covariance of the variational distribution and Jacobian structure.~}
Based on the relation in Eq. \eqref{eq:covjacob}, we expect to see similar block diagonal structures in $\Sigma^{-1}_{z|x}$ (or $\Sigma_{z|x}$) and $J_g(h(x))^\top J_g(h(x))$. Figure \ref{fig:blksz23_mnist} visualizes  these two matrices for a few test points for different values of $\beta$, latent dimensionality $d$, and block size $b$. In many cases, these two matrices stay close to each other, providing evidence for our theoretical observation.

{\bf Comparing the $\beta$-VAE and GRAE objectives.~} 
We empirically compare the values of the $\beta$-VAE  (minimization) objective \eqref{eq:elbo_nonpara}, its Taylor approximation \eqref{eq:elbo_approx2}, and the GRAE objective \eqref{eq:elbo_gauss_obs}. We train the same architecture as above on the $\beta$-VAE objective while using a single sample approximation for the expectation term $\expect_{q(z|x)} \log p(x|z)$. 
We evaluate all three objectives on a fixed held-out test set of $5000$ examples on checkpoints from the stochastic $\beta$-VAE model.
Fig.~\ref{fig:compare_obj_mnist}
shows a comparison of the three objectives for MNIST. The  objectives remain close over a wide range of $\beta$ values, with the gaps among them decreasing for smaller $\beta$ values. %
The gap between the $\beta$-VAE Taylor approximation and GRAE objectives increases more rapidly with $\beta$ than the gap between $\beta$-VAE and its Taylor approximation. This gap is analytically given by $\frac{\beta}{2}[\![\tr(I+ \frac{1}{\beta} J_g(h(x))^\top  J_g(h(x))\Sigma_{z|x}) - \log |(I+ \frac{1}{\beta} J_g(h(x))^\top  J_g(h(x))\Sigma_{z|x}| - d]\!]$, which vanishes when Eq.~\eqref{eq:covjacob} holds exactly. 
Although these gaps increase with larger $\beta$ where the %
variance of the variational     distribution increases, we observe that the trend of the three objectives is still strikingly similar (ignoring constant offsets). We provide more comparison plots in Appendix F. %

{\bf Training models using the GRAE$_{\bm{\approx}}$ objective.~}
Next, we test how the GRAE objective behaves when used for training. As discussed earlier, we use a stochastic approximation of Eq. \eqref{eq:ubreg} for the decoder regularizer to train the model (referred as GRAE$_\approx$ objective). We work with CelebA $64\times 64$ images and use standard Gaussian prior with latent dimensionality of 128, Gaussian observations, and factored Gaussian variational distribution. %
Table~\ref{tab:celebasamples_betasweep} shows the decoded samples from standard normal prior for both GRAE$_\approx$ and $\beta$-VAE trained with different values of $\beta$. It can be noted that the generated samples follow a similar trend qualitatively for both $\beta$-VAE and GRAE$_\approx$, with large $\beta$ values resulting in more blurry samples. This provides an evidence for validity of GRAE$_\approx$ as a training objective that behaves similar to $\beta$-VAE. More samples for a wider range of $\beta$ values are shown in Appendix G. %
As an attempt to quantify the similarity, we also compare the FID scores of the generated samples for both models as a function of $\beta$ in Fig.~\ref{fig:fidcomp}. The FID scores stay close for both models, particularly for $\beta\geq 0.01$ where sample quality is better (\ie, lower FID scores). This further empirically substantiates the faithfulness of our derived deterministic approximation. 
As future work, it would be interesting to explore other tractable approximations to the GRAE objective and different weights for the encoder regularizer and decoder regularizer (which are both currently fixed to $\beta/2$ in the GRAE objective).

\section{Conclusion}
We studied regularization effects in $\beta$-VAE from two perspectives: (i) analyzing the role that the choice of variational family can play in influencing the uniqueness properties of the solutions, and (ii) characterizing the regularization effect of the variational distribution  %
on the decoding model by integrating out the %
noise from the variational distribution in an approximate objective. The second perspective leads to deterministic regularized autoencoding objectives.  Empirical results confirm that the deterministic objectives are close to the original $\beta$-VAE  in terms of objective value and sample quality, further validating the analysis.  %
Our analysis helps in connecting $\beta$-VAEs and regularized autoencoders in a more principled manner, and we hope that this will motivate novel regularizers for improved sample quality and diversity in autoencoders.

{\bf Acknowledgements.~} We would like to thank Alex Alemi and Andrey Zhmoginov for providing helpful comments on the manuscript. We also thank Matt Hoffman and Kevin Murphy for insightful discussions.

\bibliography{ml}
\bibliographystyle{icml2020}

\clearpage
\appendix
{\large \bf Appendix.~} \\
Here we provide more details, in particular: \\
{\bf Appendix \ref{app:nonunique}}: non-uniqueness results for latent-variable models. \\
{\bf Appendix \ref{app:uniqueness}}: proof of Theorem \ref{thm:main} \\
{\bf Appendix \ref{app:ppca}}: experiments to empirically verify the uniqueness results \\
{\bf Appendix \ref{app:arch}}: architecture details for encoder and decoder\\
{\bf Appendix \ref{app:blockdiag}}: plots visualizing the approximate posterior precision matrix $\Sigma^{-1}_{z|x}$ and $J_g(h(x))^\top J_g(h(x))$ to verify the relation in Eq. \eqref{eq:covjacob}\\
{\bf Appendix \ref{app:compobj}}: plots comparing objective values of $\beta$-VAE and GRAE\\
{\bf Appendix \ref{app:samples}}: comparison of samples from $\beta$-VAE and GRAE$_\approx$

\section{Non-uniqueness of latent-variable models}\label{app:nonunique}
{\bf Non-uniqueness for factorized priors. } For models with continuous factorized priors, $p(z)=\prod_i p_i(z_i)$, \citet{locatello2019challenging} present a construction of transformations $r$ that mix latents but leave the marginal data distribution unchanged. 
We summarize it here for completeness and also extend it to the case of non-factored priors which was conjectured but not covered in \citep{locatello2019challenging}. %
The construction first transforms all marginal distributions $p_i(z_i)$ to uniform distributions by the mapping $u_i = F_{i}(z_i)$ where $F_{i}(\cdot)$ is the c.d.f. corresponding to the density $p_i(\cdot)$. 
This independent uniform random vector $u$ is then transformed to independent standard normal random vector by the mapping $n_i=\psi^{-1}(z_i)$ where $\psi(\cdot)$ is the c.d.f. of the standard normal. As the r.v. $n$ is independent standard normal distributed and thus spherically symmetric, for any orthogonal transformation $U$, the transformed r.v. $z'= r(z) = (F^{-1} \circ \psi \circ U\circ \psi^{-1}\circ F)(z)$ has same density as the prior $p(z)$ while completely mixing the latent variables $z$ (\ie, the transform has Jacobian
$J_r(z)$ with nonzero off-diagonals, with $|\text{det}\, J_r(z)|=1$). The marginal observation density $p(x)$ remains unchanged as well, since $p'(x)=\int p(z)p(x|r(z))dz=\int p(r^{-1}(z'))p(x|z')dz'=
\int p(z')p(x|z')dz'=
p(x)$.

{\bf Non-uniqueness for correlated priors.} Although not discussed in \citet{locatello2019challenging}, it is also possible to generalize this construction to \emph{non-factorized} priors by replacing transformations $u_i=F_{i}(z_i)$ in the first step with $u_i=F_{i|{1\ldots(i-1)}}(z_i, z_1,\ldots,z_{i-1})$ (and its inverse in the last step),  where $F_{i|{1\ldots(i-1)}}$ is the c.d.f. of the conditional distribution of $z_i$ given \emph{earlier} latents (with arbitrary ordering on the latents) \citep{darmois1951analyse,hyvarinen1999nonlinear}.

\section{Uniqueness in $\beta$-VAE}\label{app:uniqueness}
\setcounter{theorem}{0}
We provide the proof of  Theorem~\ref{thm:main} which states that any $\beta$-VAE solution obtained by transforming the generative model of the optimal solution while keeping the prior and the marginal log-likelihood same, will result in decreasing the objective value. 

\begin{theorem}
Let $R$ be the class of mixing transforms that leave the prior invariant, and $p_r$ be the generative model obtained by transforming the latents of the model $p$ by $r\in R$. %
Let $L(q^*, p^*) = \max_{q \in \mcal{Q}, p \in \mcal{P}} L(q, p)$.
Let $\mcal{\widetilde{Q}}$ be the completion of $\mcal{Q}$ by $R$, \ie, $\mcal{\widetilde{Q}}=\{r\#q: q\in\mcal{Q}, r\in R\}\cup \mcal{Q}$.  If $\mcal{Q}$ is s.t. $q \in \mcal{Q} \Rightarrow r\#q \not\in \mcal{Q}$ for all $r\in R$, and $\argmax_{q\in\mcal{\widetilde{Q}}}L(q,p^*)$ is unique, then $\max_{q \in \mcal{Q}} L(q, p_r^*) < L(q^*, p^*)$.
\label{thm:main_app}
\end{theorem}
\begin{proof}
The objective after transforming a model $p$ with $r\in R$ yields:
\begin{align}
\begin{split}
& L(q, p_r) := \expect_{q(z|x)} \log p_{x|z}(x|r(z)) - \beta\, \kl(q(z|x)\Vert p_z(z)) \\
&=\int \log\left( p_{x|z}(x|r(z)) \left(\frac{p_z(z)}{q(z|x)}\right)^\beta\right) q(z|x) dz \\
&=\int \log\left( p_{x|z}(x|z') \left(\frac{p_z(r^{-1}(z'))}{q(r^{-1}(z')|x)}\right)^\beta\right) q(r^{-1}(z')|x) dz' \\
&=\int \log\left( p_{x|z}(x|z) \left(\frac{p_z(z)}{q(r^{-1}(z)|x)}\right)^\beta\right) q(r^{-1}(z)|x) dz \\
&=\expect_{q_{r}} \log p_{x|z}(x|z) - \beta\, \kl(q_{r}\Vert p_z) = L(q_r, p).
\end{split}
\label{eq:elbo_r_app}
\end{align}
where we used the shorthand $q_{r}$ to denote the transformed variational density $r\# q_{z|x}$.  %
This also implies that $L(q,p)=L(q_{r^{-1}}, p_r)$. 
Hence, the optimum value of the $\beta$-VAE objective $L(q^*, p^*) = L(q_{r^{-1}}^*, p_r^*)$ $\forall r\in R$. 
Choosing $\mcal{Q}$ such that $q \in \mcal{Q} \Rightarrow r\#q \not\in \mcal{Q}$ $\forall r\in R$ will ensure that $q_{r^{-1}}^*\notin \mcal{Q}$. 
Note that optimizing the $\beta$-VAE objective over $\mcal{\widetilde{Q}}=\{r\# q: q\in\mcal{Q}, r\in R\}$ does not change the optimal value, \ie, $L(q^*,p^*)= \max_{q\in\mcal{Q},p\in\mcal{P}} L(q,p) = \max_{q\in\mcal{\widetilde{Q}},p\in\mcal{P}} L(q,p)$. 
The statement of the theorem can now be proven by contradiction. Suppose $\exists q'\in \mcal{Q}$
s.t. $L(q', p_r^*)=L(q'_r, p^*) \geq L(q^*, p^*)$. 
However, this will contradict the assumption that $q^* =\argmax_{q\in\mcal{\widetilde{Q}}} L(q,p^*)$ and that it is unique. 
Hence $\forall q'\in\mcal{Q}$, $L(q',p_r^*)< L(q^*,p^*)$.
\end{proof}

\begin{corollary}
If the set $\mcal{\widetilde{Q}}$ defined in Theorem~\ref{thm:main} is convex, then $\argmax_{q\in\mcal{\widetilde{Q}}}L(q,p^*)$ is unique. When all other assumptions of Theorem~\ref{thm:main} are met, we have $\max_{q \in \mcal{Q}} L(q, p_r^*) < L(q^*, p^*)$.
\end{corollary}
\begin{proof}
As the $\beta$-VAE objective is strongly concave in $q$ (first term is linear in $q$ and the KL-divergence term is strongly convex in $q$ \citep{adamvcik2014information,sason2016f}), optimizing it over a convex set $\mcal{\widetilde{Q}}$ for a fixed $p$ will have a unique maximizer. Hence, the result of  Theorem~\ref{thm:main} holds when all other assumptions are met.
\end{proof}

\section{Uniqueness in $\beta$-VAE: empirical results}\label{app:ppca}
\begin{figure*}[t]
\centering
\includegraphics[width=0.35\textwidth]{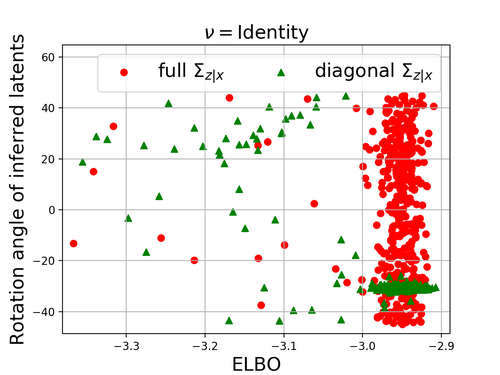}
\includegraphics[width=0.35\textwidth]{figures/ppca/10507739_amortized_Arandn_Crotinit_tanh_calcsimple_sigmamfmatrixdiag.png}
\includegraphics[width=0.35\textwidth]{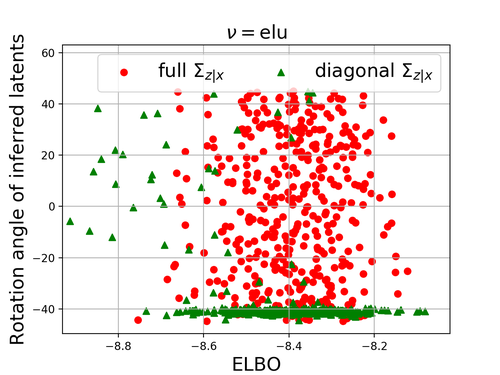}
\includegraphics[width=0.35\textwidth]{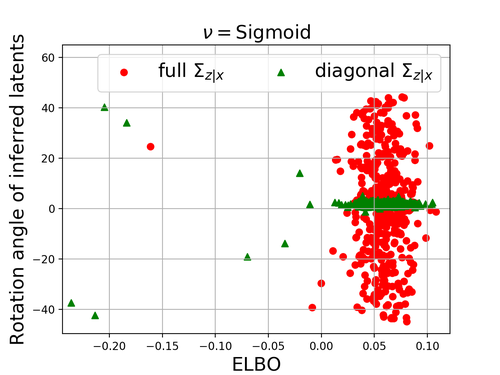}
\caption{Rotation angle between true latents and inferred latents with VAEs (multiple runs): The true generative model for the data is $x = \nu(Wz) + \epsilon$, with $z\sim\mcal{N}(0,I_{2\times 2}), \epsilon\sim\mcal{N}(0,.05^2 I_{2\times 2})$. Diagonal $\Sigma$ encourages unique solution in terms of inferred latents as the ELBO increases. Please refer to the text for more details.}
\label{fig:ppca_am}
\end{figure*}
\begin{figure*}[h]
\centering
\includegraphics[width=0.35\textwidth]{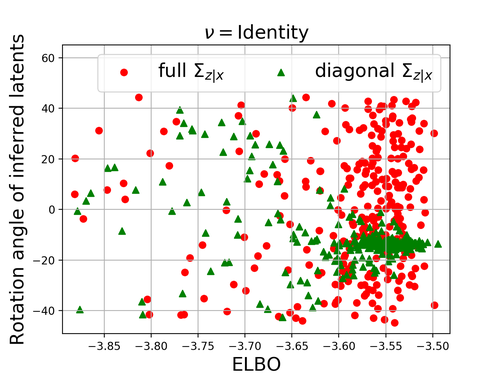}
\includegraphics[width=0.35\textwidth]{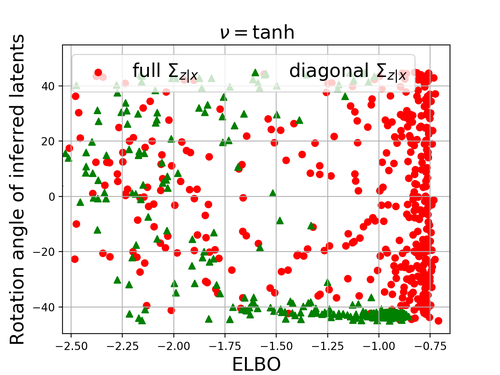}
\includegraphics[width=0.35\textwidth]{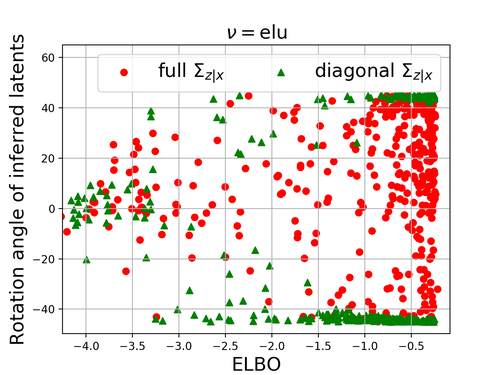}
\includegraphics[width=0.35\textwidth]{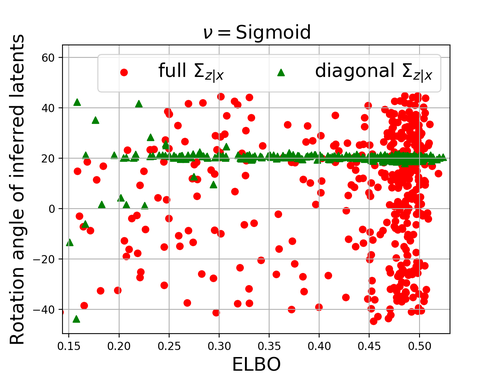}
\caption{Rotation angle between true latents and inferred latents with unamortized variational inference (multiple runs): 
The true generative model for the data is $x = \nu(Wz) + \epsilon$, with $z\sim\mcal{N}(0,I_{2\times 2}), \epsilon\sim\mcal{N}(0,.05^2 I_{2\times 2})$. Diagonal $\Sigma$ encourages unique solution in terms of inferred latents as the ELBO increases. Please refer to the text for more details.}
\label{fig:ppca_unam}
\end{figure*}

Here we provide more details on simulations we perform to verify that appropriately choosing the variational family $\mcal{Q}$ imparts uniqueness to the $\beta$-VAE (or plain VI) solution with respect to the corresponding class of transformations $R$. We consider the case of $\mcal{Q}$ being mean-field distributions and $R$ being rotations. We work with two dimensional data so that rotations can be represented by a scalar (angle of planar rotation) which can be plotted conveniently. 

We generate $4$k samples in two dimensions by the generative model $x = \nu(Wz) + \epsilon$, with $z\sim\mcal{N}(0,I_{2\times 2}), \epsilon\sim\mcal{N}(0,.05^2 I_{2\times 2})$. Entries of $W\in\reals^{2\times 2}$ are sampled from $\mcal{N}(0,I_{2\times 2})$ and normalized such that columns have different $\ell_2$ norms (2 and 1, respectively). The function $\nu$ acts elementwise on the vector and is taken to be either Identity (reducing it to Probabilistic PCA) or a nonlinearity (tanh, sigmoid, elu). A VAE is learned on this data, with the decoder of the same form as the generative model, \ie, $x = \nu(Az)+b$, and variational posterior of the form $q(z|x)=\mcal{N}(Cx+d,\Sigma)$ with $A,b,C,d$ and $\Sigma$ learned.

We also do unamortized variational inference with the variational distribution $q_i(z)=\mcal{N}(\mu_i, \Sigma_i)$ for the $i$'th data point ($\mu_i\in\reals^2, \Sigma_i\in S^2_+$). For the case of mean-field variational inference, we learn only diagonal elements of $\Sigma$ or $\Sigma_i$ that appear in the objective (ELBO), leaving the off-diagonal elements free. 

We do multiple runs %
by randomizing over $W, \epsilon, z$ for the true generative model of data and randomizing over the initializations of the VAE parameters. For each run, we measure the rotation angle between the true latents and the inferred posterior means. Rotation angle is computed after ignoring permutation and reflection of axes so the angle ranges in $(-45,45]$  degrees with $-45$ degree rotation being same as $45$ degree rotation. 
The rotation angle is estimated by solving the problem $\min_O \lVert Z - OM\rVert_F^2$ where $O$ is an orthogonal matrix, $Z\in\reals^{2\times 4000}$ is the matrix of original sampled latents used for generating the data, $M\in\reals^{2\times 4000}$ is the matrix of inferred latents (the means of $q(z|x_i)$). This is known as Orthogonal Procrustes Problem whose solution is given by $O=UV^\top$ where $U$ and $V$ are left and right singular vectors of $ZM^\top$. We then extract axes permutation, axes reflection and the rotation angle from the matrix $O$.

We train the model parameters by gradient descent (use all data samples for a gradient update). The learning rate is set to $0.1$ in the beginning and is reduced in a stepwise manner after every $n$ iterations by a factor of $10$, till the learning rate reaches $10^{-5}$ (\ie, each model is trained for $5n$ iterations). To train models up to varying degree of optimality, we take $n$ to be $\{100,500,1\text{k},2\text{k},3\text{k},4\text{k},5\text{k},6\text{k},7\text{k},8\text{k},9\text{k},10\text{k},20\text{k}\}$. We do $20$ random trials for each value of $n$ from $100$ to $9$k, and $100$ random trials for $n=10$k and $n=20$k (total $420$ random trials per experiment).

We also take special care in how we parameterize the models so that the implicit bias of gradient descent does not affect the solution in any particular way. For amortized inference with VAE, we parameterize encoder parameters as $C=O_1C'$ and $\Sigma=O_2D'D'^\top O_2^\top$, where $O_1, O_2$ are random $2\times 2$ rotation matrices that are fixed at the beginning of training (different for each random trial), and $C'$ and $D'$ are the parameters learned by gradient descent. 
For unamortized variational inference, we again parameterize $\Sigma_i=O D_i'D_i'^\top O^\top$ with $O$ being a random $2\times 2$ rotation matrix and $D_i'$ being the trainable parameters. We notice that without these reparameterizations, implicit bias of gradient descent can impart some degree of uniqueness to the model, which we particularly noticed in the case of unamortized inference. 

Fig.~\ref{fig:ppca_am} shows the scatter plots of rotation angle vs ELBO for the amortized case (VAE) with $\nu=$identity,tanh,elu, sigmoid. Having diagonal $\Sigma$ (mean-field $q(z|x)$) encourages unique solution in terms of inferred latents as the ELBO increases (ignoring the permutations and reflections of latent axes). Fig.~\ref{fig:ppca_unam} shows the scatter plots of rotation angle vs ELBO for the unamortized case, with similar observations regarding uniqueness.

\clearpage
\twocolumn
\section{Architectures}\label{app:arch}
Here we describe the architecture for our MNIST and CelebA models.
We list Conv (convoutional) and ConvT (transposed convolution) layers with their number of filters, kernel size, and stride.
Latent dimension for CelebA is fixed to be 128 in all experiments, and the dimension for MNIST experiments is varied in the set $\{8,12,16,20\}$. 
The architecture is shown for the GRAE model. For $\beta$-VAE, the encoder has an extra output through an fully-connected layer for standard deviations of the posterior factors in the case of mean-field posterior, or for the elements of the Cholesky factor in the case of block-diagonal posterior covariance matrices. 

\begin{table}[h]
\centering
\begin{small}
    \begin{tabular}{p{0.45\linewidth}  p{0.45\linewidth}}
    	\toprule
        MNIST &  CelebA \\
		\midrule
        
        \multicolumn{2}{c}{Encoder} \\ \midrule

        $x \in \mathcal{R}^{28{\times}28}$ \newline
        $\rightarrow \text{Conv}_{64,4,1} \rightarrow \text{BN} \rightarrow \text{ELU}$\newline
        $\quad \rightarrow \text{Conv}_{128,4,2}\rightarrow \text{BN} \rightarrow \text{ELU}$\newline
        $\quad \rightarrow \text{Conv}_{256,4,2}\rightarrow \text{BN} \rightarrow \text{ELU}$\newline
        $\quad \rightarrow \text{Conv}_{512,4,2}\rightarrow \text{BN} \rightarrow \text{ELU}$\newline
        $\quad \rightarrow \text{Conv}_{512,4,1}\rightarrow \text{BN} \rightarrow \text{ELU}$\newline
        $\quad \rightarrow \text{Flatten} \rightarrow \text{FC}_{d}$
      & $x \in \mathcal{R}^{64{\times}64\times 3}$ \newline
        $\rightarrow \text{Conv}_{128,5,1} \rightarrow \text{BN} \rightarrow \text{ELU}$\newline
        $\quad \rightarrow \text{Conv}_{256,5,2}\rightarrow \text{BN} \rightarrow \text{ELU}$\newline
        $\quad \rightarrow \text{Conv}_{512,5,2}\rightarrow \text{BN} \rightarrow \text{ELU}$\newline
        $\quad \rightarrow \text{Conv}_{1024,5,2}\rightarrow \text{BN} \rightarrow \text{ELU}$\newline
        $\quad \rightarrow \text{Conv}_{1024,5,2}\rightarrow \text{BN} \rightarrow \text{ELU}$\newline
        $\quad \rightarrow \text{Flatten} \rightarrow \text{FC}_{128}$\\

\midrule
\multicolumn{2}{c}{Decoder} \\ 
\midrule
        $z \in \mathcal{R}^{d} \rightarrow \text{FC}_{7{\times}7{\times}256}$\newline
        $\rightarrow \text{BN} \rightarrow \text{ELU}$\newline
        $\rightarrow \text{ConvT}_{512,4,1}\rightarrow \text{BN} \rightarrow \text{ELU}$\newline
        $\rightarrow \text{ConvT}_{256,4,1}\rightarrow \text{BN} \rightarrow \text{ELU}$\newline
        $\rightarrow \text{ConvT}_{128,4,2}\rightarrow \text{BN} \rightarrow \text{ELU}$\newline
        $\rightarrow \text{ConvT}_{64,4,2}\rightarrow \text{ELU}$\newline
        $\rightarrow \text{Conv}_{1,4,1}\rightarrow \text{Sigmoid}$
      & $z \in \mathcal{R}^{128} \rightarrow \text{FC}_{16{\times}16{\times}512}$\newline
        $\rightarrow \text{BN} \rightarrow \text{ELU}$\newline
        $\rightarrow \text{ConvT}_{1024,5,1} {\tiny \rightarrow \text{BN} \rightarrow \text{ELU}}$\newline
        $\rightarrow \text{ConvT}_{512,5,2}\rightarrow \text{BN} \rightarrow \text{ELU}$\newline
        $\rightarrow \text{ConvT}_{256,5,2}\rightarrow \text{BN} \rightarrow \text{ELU}$\newline
        $\rightarrow \text{ConvT}_{128,5,2}\rightarrow \text{ELU}$\newline
        $\rightarrow \text{Conv}_{3,5,1}\rightarrow \text{Sigmoid}$\\
  \bottomrule
    \end{tabular}
    \end{small}
\end{table}

\section{Relation between %
covariance of the variational distribution and Jacobian structure}\label{app:blockdiag}
We provide more plots visualizing the %
variational distribution precision matrix $\Sigma^{-1}_{z|x}$ and $J_g(h(x))^\top J_g(h(x))$ in Figure \ref{fig:blksz23_mnist_app}. Based on the relation in Eq. \eqref{eq:covjacob}, we expect to see similar block diagonal structures in $\Sigma^{-1}_{z|x}$ (or $\Sigma_{z|x}$) and $J_g(h(x))^\top J_g(h(x))$. 

\begin{figure}[t]
\centering
\subfigure[$\beta=0.2,d=8,b=2$]{\label{fig:blksz2_a}\includegraphics[width=0.13\textwidth]{figures/jjt_plots_blk2/enccov_inv-it-9800_10-lr=0.001,block_diag_size=2,layer_width=64,latent_dim=8,beta=0.2__1.png}~~ 
\includegraphics[width=0.13\textwidth]{figures/jjt_plots_blk2/JJT-it-9800_10-lr=0.001,block_diag_size=2,layer_width=64,latent_dim=8,beta=0.2.png}}\\
\subfigure[$\beta=0.4,d=8,b=2$]{\label{fig:blksz2_b}\includegraphics[width=0.13\textwidth]{figures/jjt_plots_blk2/enccov_inv-it-9800_18-lr=0.001,block_diag_size=2,layer_width=64,latent_dim=8,beta=0.4.png}~~ 
\includegraphics[width=0.13\textwidth]{figures/jjt_plots_blk2/JJT-it-9800_18-lr=0.001,block_diag_size=2,layer_width=64,latent_dim=8,beta=0.4.png}} \\
\subfigure[$\beta=0.6,d=12,b=2$]{\label{fig:blksz2_c}\includegraphics[width=0.13\textwidth]{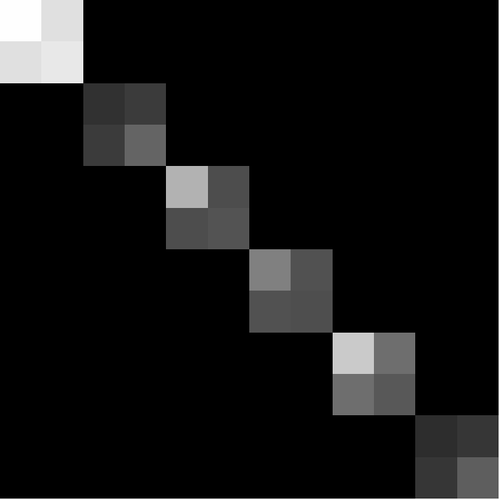}~~ 
\includegraphics[width=0.13\textwidth]{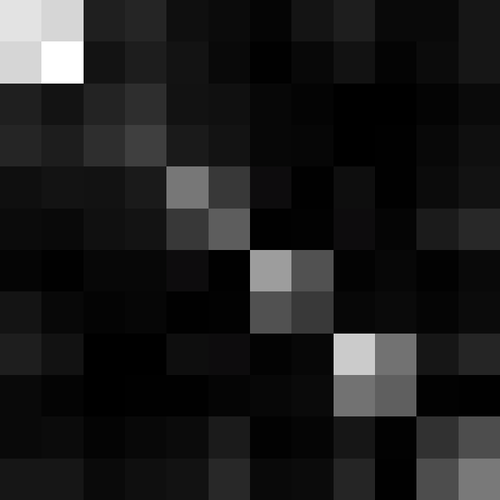}}\\
\subfigure[$\beta=1,d=8,b=2$]{\label{fig:blksz2_d}\includegraphics[width=0.13\textwidth]{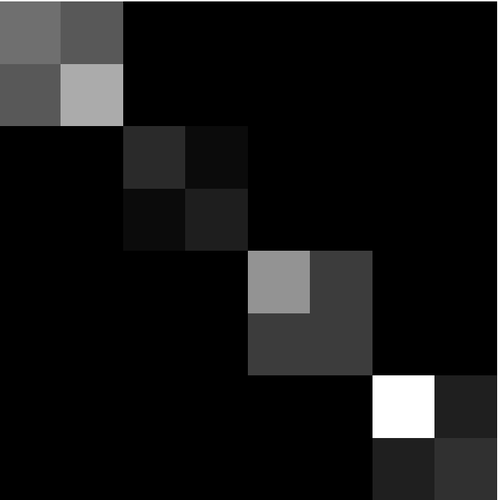}~~ 
\includegraphics[width=0.13\textwidth]{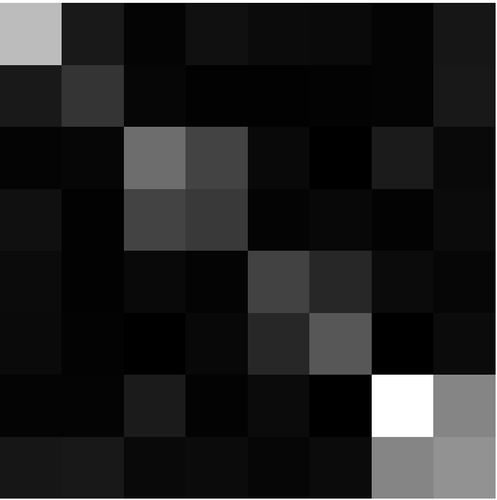}}\\
\subfigure[$\beta=0.4,d=12,b=3$]{\label{fig:blksz3_a}\includegraphics[width=0.13\textwidth]{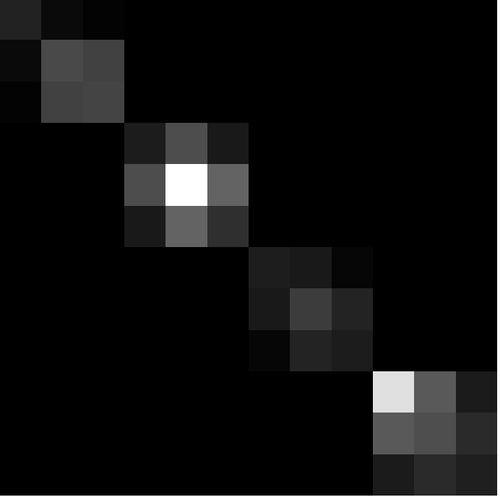}~~ 
\includegraphics[width=0.13\textwidth]{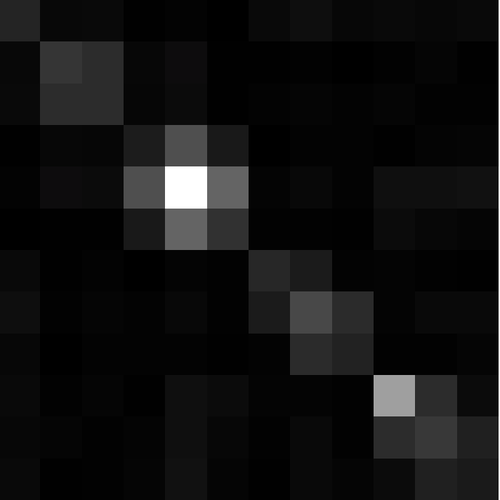}}\\
\subfigure[$\beta=0.6,d=8,b=3$]{\label{fig:blksz3_b}\includegraphics[width=0.13\textwidth]{figures/jjt_plots_blk3/enccov_inv_34-beta=0.6,lr=0.001,block_diag_size=3,latent_dim=8,layer_width=64.png}~~
\includegraphics[width=0.13\textwidth]{figures/jjt_plots_blk3/JJT_34-beta=0.6,lr=0.001,block_diag_size=3,latent_dim=8,layer_width=64.png}}\\
\subfigure[$\beta=0.8,d=20,b=3$]{\label{fig:blksz3_c}\includegraphics[width=0.13\textwidth]{figures/jjt_plots_blk3/enccov_inv_48-beta=0.8,lr=0.001,block_diag_size=3,latent_dim=20,layer_width=64.png}~~
\includegraphics[width=0.13\textwidth]{figures/jjt_plots_blk3/JJT_48-beta=0.8,lr=0.001,block_diag_size=3,latent_dim=20,layer_width=64.png}}

\caption{MNIST: Precision matrices $\Sigma_{z|x}^{-1}$ of variational distributions (left) and the corresponding $J_g(h(x))^\top J_g(h(x))$ (right), for different values of $\beta$ and latent dimensionality $d$. Block size $b$ for the posterior covariance was taken to be 2 (top row) or 3 (bottom row).}
\label{fig:blksz23_mnist_app}
\end{figure}

\clearpage
\onecolumn
\section{Comparison of objective values}\label{app:compobj}
We show more plots comparing objective values of $\beta$-VAE, its Taylor series approximation, and the GRAE objective for various values of $\beta$ on MNIST. The latent dimension $d$ is fixed to 12, and the block-size $b$ (for block-diagonal posterior covariance) is varied over $\{1,2,3\}$. The model is trained using the $\beta$-VAE objective and is evaluated on all objectives using a held-out test set of 5000 examples. As $\beta$ increases, the gap between the $\beta$-VAE objective and its Taylor approximation increases. The gap between the Taylor approximation and the GRAE objective also increases with $\beta$. This gap is analytically given by $\frac{\beta}{2}[\tr(I+ \frac{1}{\beta} J_g(h(x))^\top  J_g(h(x))\Sigma_{z|x}) - \log |(I+ \frac{1}{\beta} J_g(h(x))^\top  J_g(h(x))\Sigma_{z|x}| - d]$. %

\begin{figure*}[h]
\includegraphics[width=0.35\textwidth]{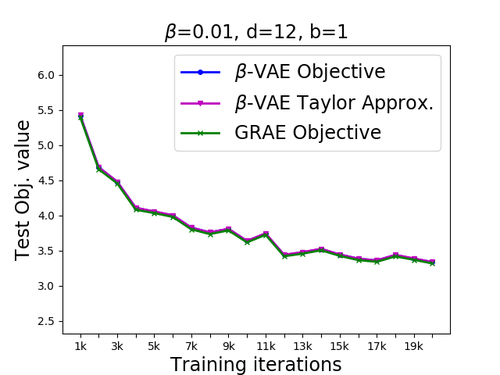}\hspace{-5mm}
\includegraphics[width=0.35\textwidth]{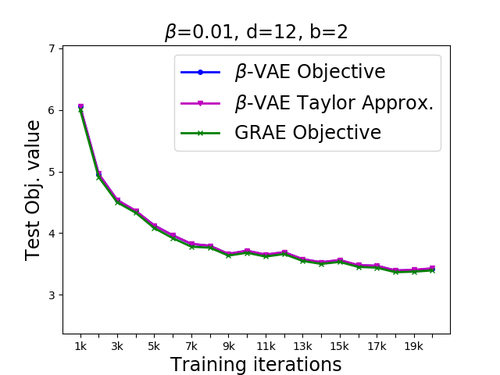}\hspace{-5mm}
\includegraphics[width=0.35\textwidth]{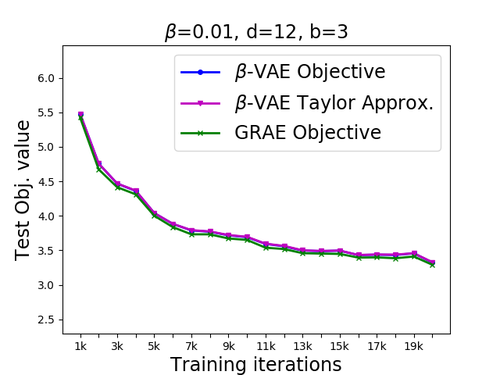}
\includegraphics[width=0.35\textwidth]{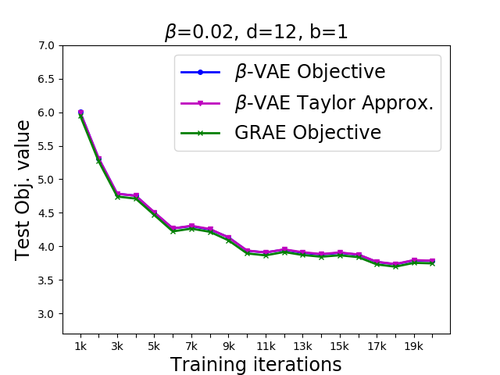}\hspace{-5mm}
\includegraphics[width=0.35\textwidth]{supplementary/compare_obj/43-lr=0.0001,block_diag_size=2,layer_width_enc=64,latent_dim=12,beta=0.02__elbo_approx_compare.npz.png}\hspace{-5mm}
\includegraphics[width=0.35\textwidth]{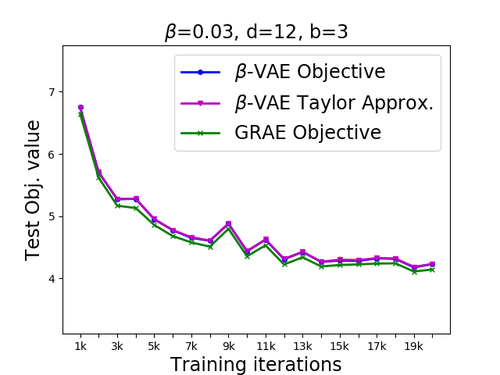}
\includegraphics[width=0.35\textwidth]{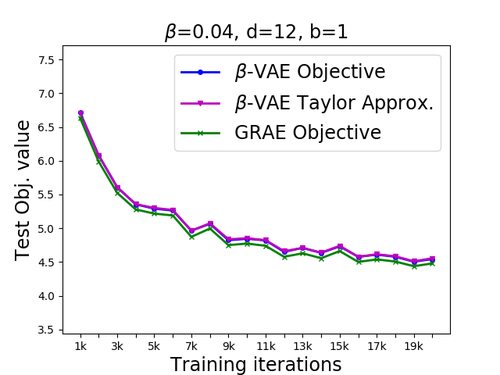}\hspace{-5mm}
\includegraphics[width=0.35\textwidth]{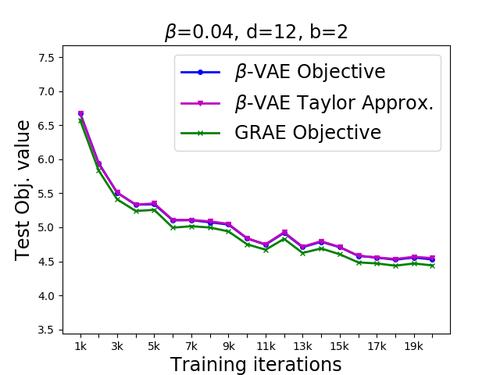}\hspace{-5mm}
\includegraphics[width=0.35\textwidth]{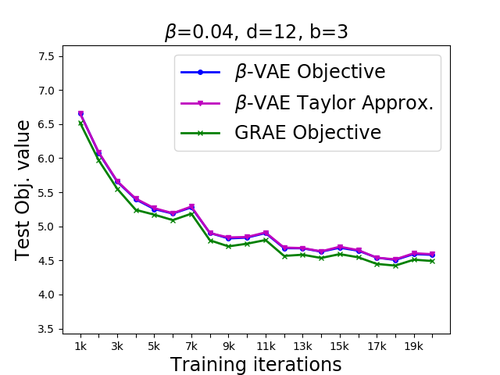}

\end{figure*}

\begin{figure*}
\includegraphics[width=0.35\textwidth]{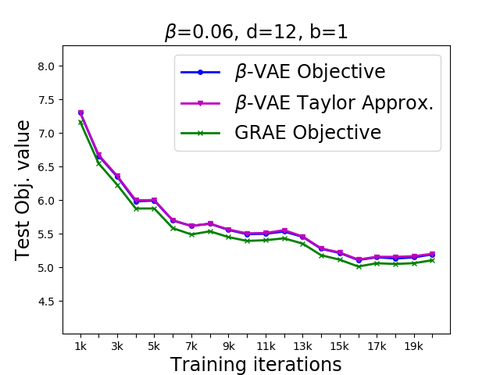}\hspace{-5mm}
\includegraphics[width=0.35\textwidth]{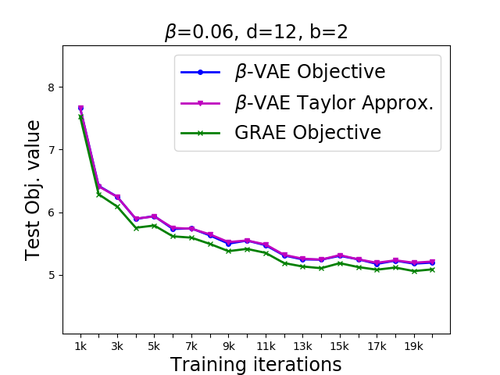}\hspace{-5mm}
\includegraphics[width=0.35\textwidth]{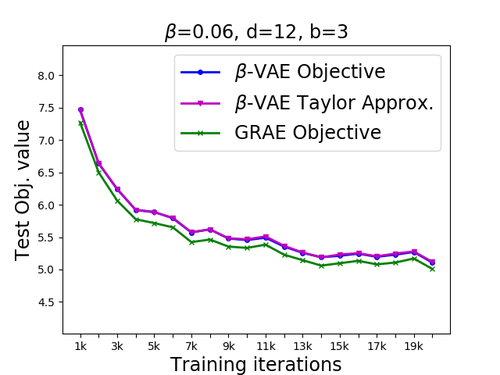}

\includegraphics[width=0.35\textwidth]{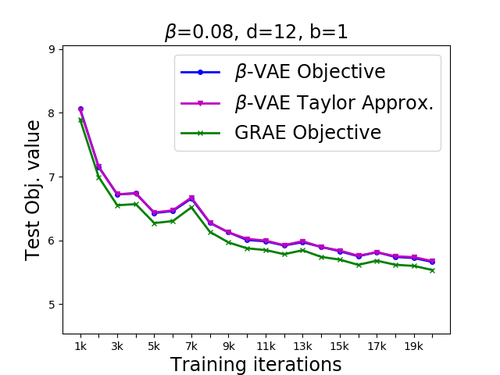}\hspace{-5mm}
\includegraphics[width=0.35\textwidth]{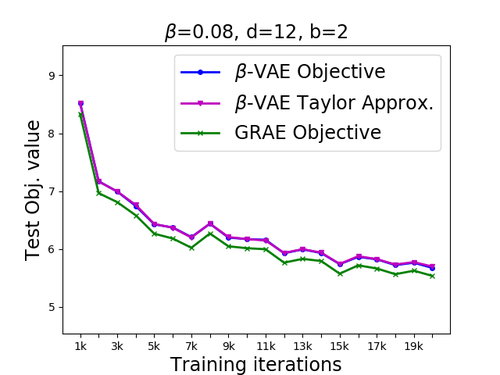}\hspace{-5mm}
\includegraphics[width=0.35\textwidth]{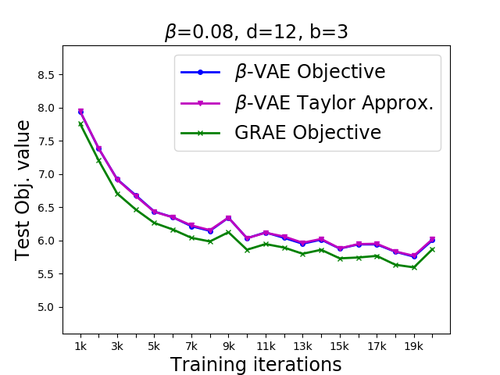}
\includegraphics[width=0.35\textwidth]{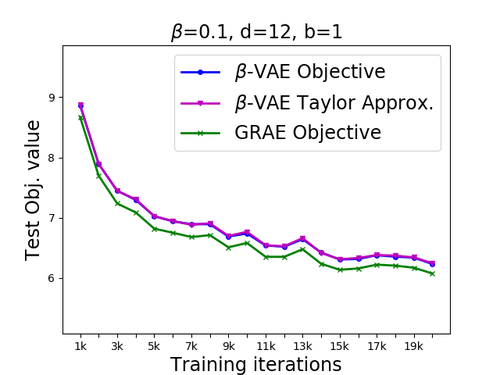}\hspace{-5mm}
\includegraphics[width=0.35\textwidth]{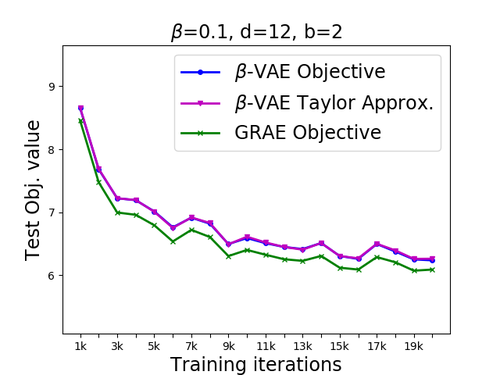}\hspace{-5mm}
\includegraphics[width=0.35\textwidth]{supplementary/compare_obj/142-lr=0.0001,block_diag_size=3,layer_width_enc=64,latent_dim=12,beta=0.1__elbo_approx_compare.npz.png}
\includegraphics[width=0.35\textwidth]{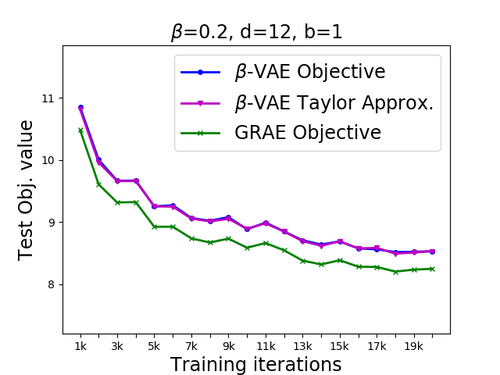}\hspace{-5mm}
\includegraphics[width=0.35\textwidth]{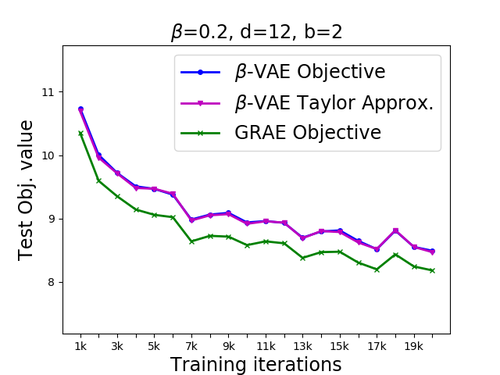}\hspace{-5mm}
\includegraphics[width=0.35\textwidth]{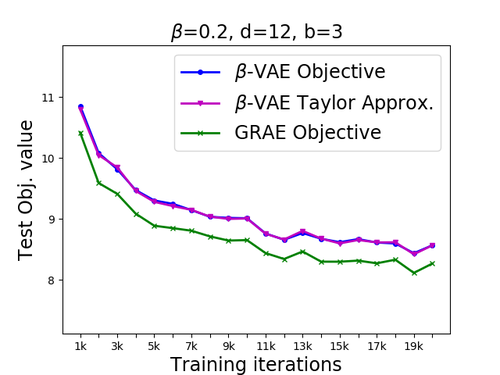}
\end{figure*}

\begin{figure*}[t]
\includegraphics[width=0.35\textwidth]{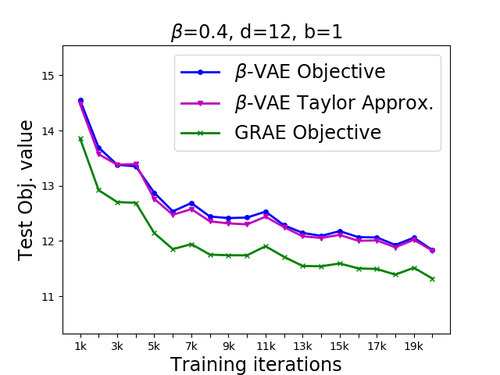}\hspace{-5mm}
\includegraphics[width=0.35\textwidth]{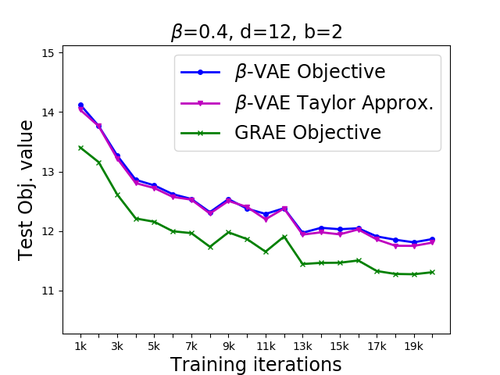}\hspace{-5mm}
\includegraphics[width=0.35\textwidth]{supplementary/compare_obj/178-lr=0.0001,block_diag_size=3,layer_width_enc=64,latent_dim=12,beta=0.4__elbo_approx_compare.npz.png}

\includegraphics[width=0.35\textwidth]{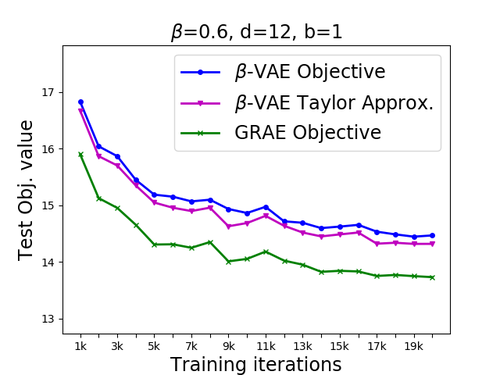}\hspace{-5mm}
\includegraphics[width=0.35\textwidth]{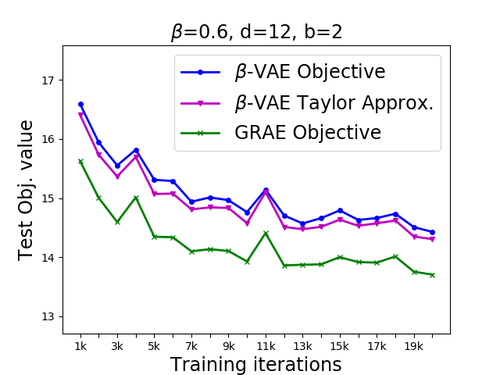}\hspace{-5mm}
\includegraphics[width=0.35\textwidth]{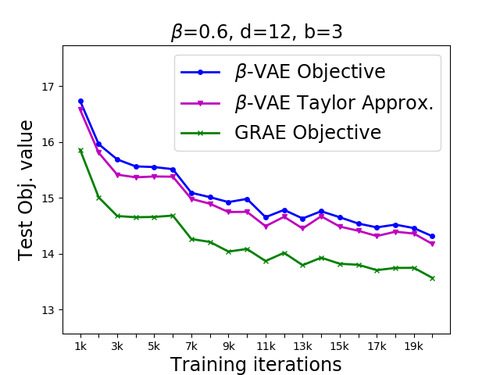}
\includegraphics[width=0.35\textwidth]{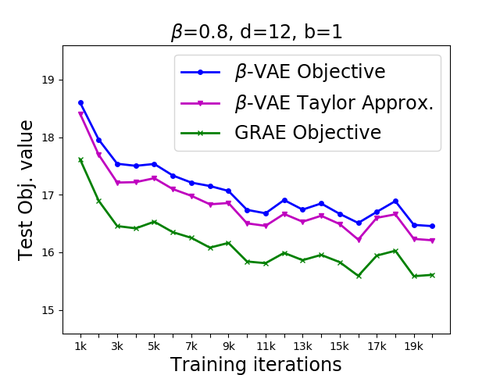}\hspace{-5mm}
\includegraphics[width=0.35\textwidth]{supplementary/compare_obj/223-lr=0.0001,block_diag_size=2,layer_width_enc=64,latent_dim=12,beta=0.8__elbo_approx_compare.npz.png}\hspace{-5mm}
\includegraphics[width=0.35\textwidth]{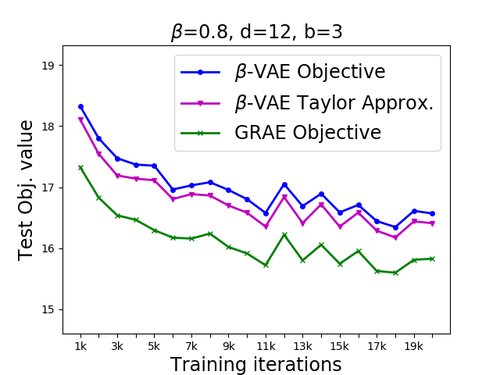}
\includegraphics[width=0.35\textwidth]{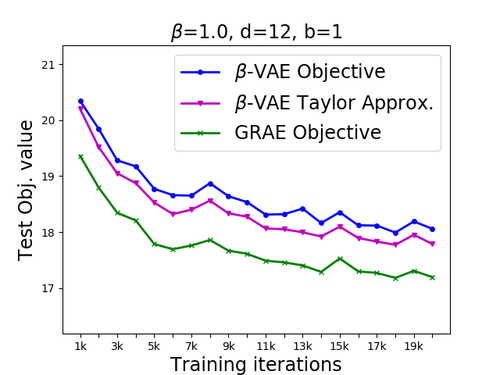}\hspace{-5mm}
\includegraphics[width=0.35\textwidth]{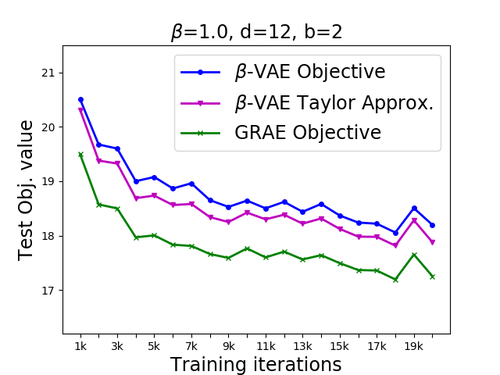}\hspace{-5mm}
\includegraphics[width=0.35\textwidth]{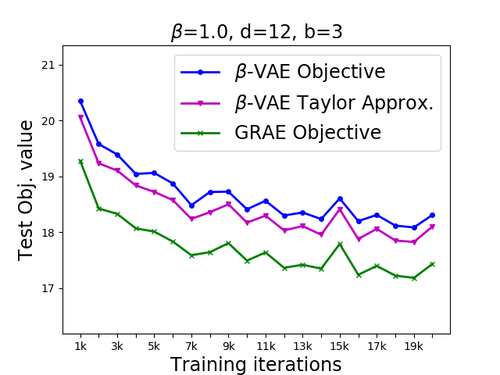}
\end{figure*}

\clearpage
\onecolumn
\vspace{5mm}
\section{Samples from VAE and GRAE}\label{app:samples}
We visualize random CelebA samples using the standard normal prior for $\beta$-VAE and Gaussian-RAE for different values of $\beta$ below. %
Note that for training the GRAE model, we actually minimize the stochastic approximation of the upper bound of Eq. (17) for the decoder regularizer. Hence we will refer to it as {\bf GRAE$_{\approx}$}. 
Samples from models trained using the GRAE objective have similar smoothness/blurriness as samples from $\beta$-VAE models, particularly for small $\beta$ values.

\begin{tabular}{lcccc}
& $\bm{\beta=0.02}$ & $\bm{\beta=0.04}$   \\
\rotatebox[origin=lt]{90}{~~~~~~~~~~~~~~~~~~~~~~~~{\bf $\bm{\beta}$-VAE} Samples}\hspace{-4mm} & \includegraphics[width=0.48\textwidth]{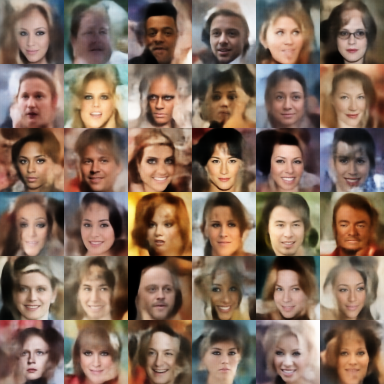}\hspace{-2mm} &
\includegraphics[width=0.48\textwidth]{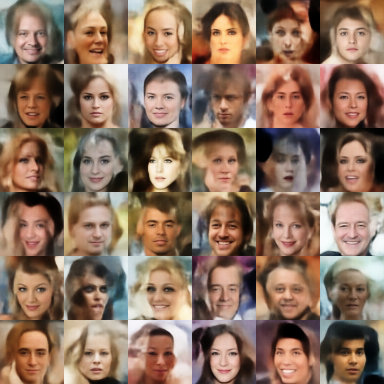}\hspace{-2mm}  \vspace{2mm}\\
\rotatebox[origin=lt]{90}{~~~~~~~~~~~~~~~~~~~~~~~~{\bf $\bm{\approx}$GRAE} Samples}\hspace{-4mm} & \includegraphics[width=0.48\textwidth]{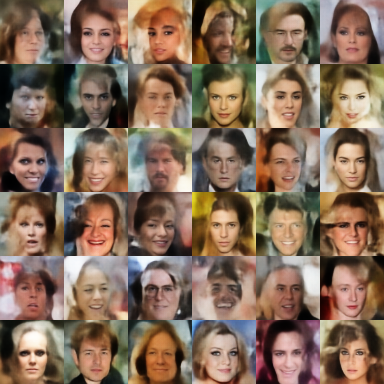}\hspace{-2mm} &
\includegraphics[width=0.48\textwidth]{figures/celeba_rae_orth0/4-lr=0.001,beta=0.02,lambda_orth=0.0__it49999_0_size_6x6_decoder_samples_stdnormal.png}
\end{tabular}

\clearpage

\begin{tabular}{lcccc}
& $\bm{\beta=0.06}$ & $\bm{\beta=0.08}$   \\
\rotatebox[origin=lt]{90}{~~~~~~~~~~~~~~~~~~~~~~~~{\bf $\bm{\beta}$-VAE} Samples}\hspace{-4mm} & \includegraphics[width=0.48\textwidth]{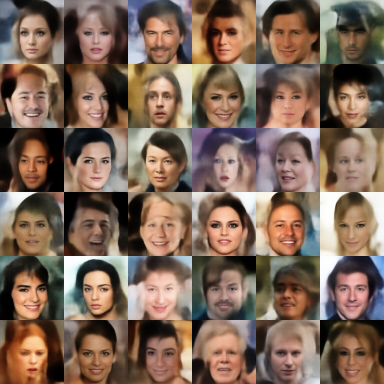}\hspace{-2mm} &
\includegraphics[width=0.48\textwidth]{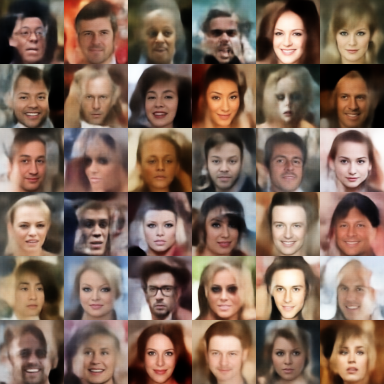}\hspace{-2mm}  \vspace{2mm}\\
\rotatebox[origin=lt]{90}{~~~~~~~~~~~~~~~~~~~~~~~~{\bf GRAE$_{\approx}$} Samples}\hspace{-4mm} & \includegraphics[width=0.48\textwidth]{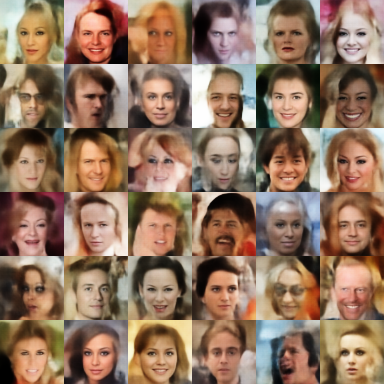}\hspace{-2mm} &
\includegraphics[width=0.48\textwidth]{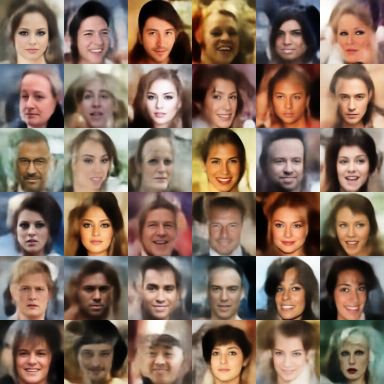}
\end{tabular}

\clearpage

\begin{tabular}{lcccc}
& $\bm{\beta=0.1}$ & $\bm{\beta=0.2}$   \\
\rotatebox[origin=lt]{90}{~~~~~~~~~~~~~~~~~~~~~~~~{\bf $\bm{\beta}$-VAE} Samples}\hspace{-4mm} & \includegraphics[width=0.48\textwidth]{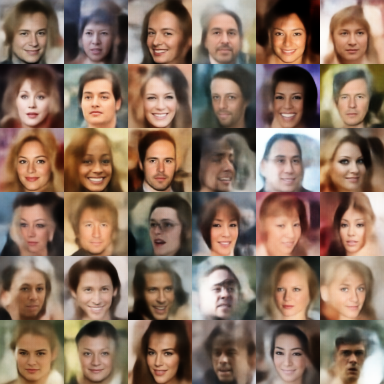}\hspace{-2mm} &
\includegraphics[width=0.48\textwidth]{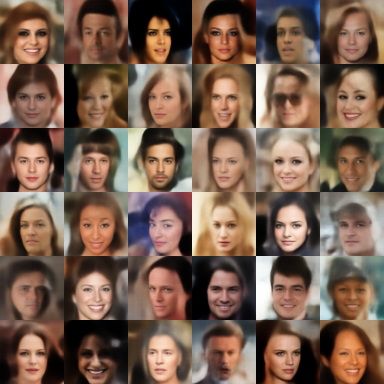}\hspace{-2mm}  \vspace{2mm}\\
\rotatebox[origin=lt]{90}{~~~~~~~~~~~~~~~~~~~~~~~~{\bf GRAE$_{\approx}$} Samples}\hspace{-4mm} & \includegraphics[width=0.48\textwidth]{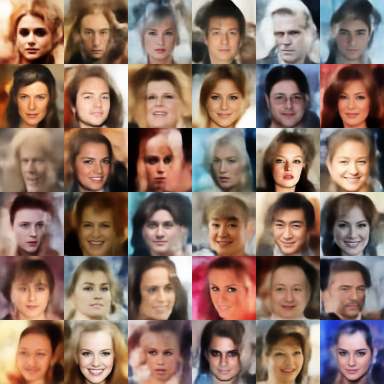}\hspace{-2mm} &
\includegraphics[width=0.48\textwidth]{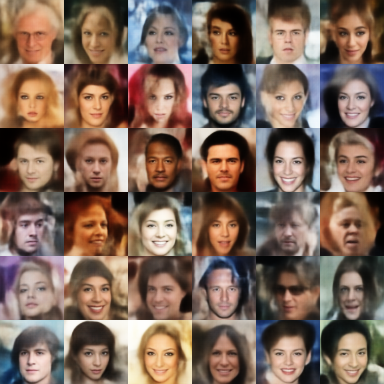}
\end{tabular}

\clearpage

\begin{tabular}{lcccc}
& $\bm{\beta=0.4}$ & $\bm{\beta=0.6}$   \\
\rotatebox[origin=lt]{90}{~~~~~~~~~~~~~~~~~~~~~~~~{\bf $\bm{\beta}$-VAE} Samples}\hspace{-4mm} & \includegraphics[width=0.48\textwidth]{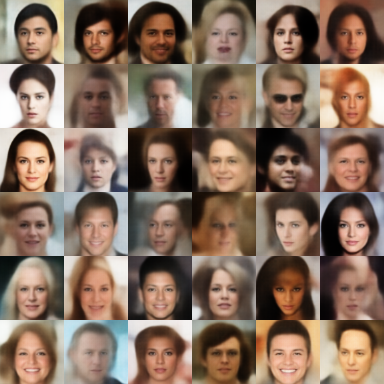}\hspace{-2mm} &
\includegraphics[width=0.48\textwidth]{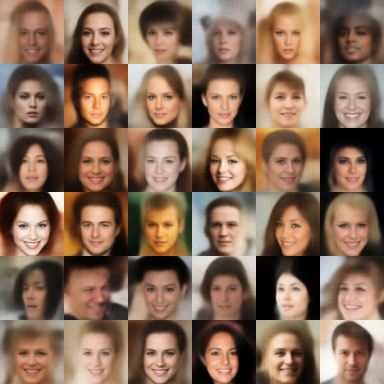}\hspace{-2mm}  \vspace{2mm}\\
\rotatebox[origin=lt]{90}{~~~~~~~~~~~~~~~~~~~~~~~~{\bf GRAE$_{\approx}$} Samples}\hspace{-4mm} & \includegraphics[width=0.48\textwidth]{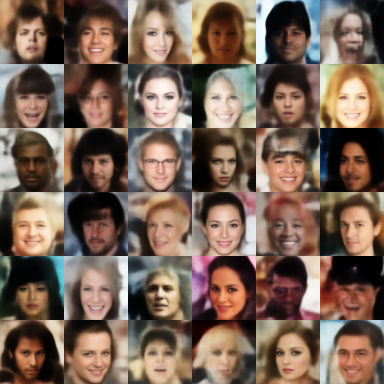}\hspace{-2mm} &
\includegraphics[width=0.48\textwidth]{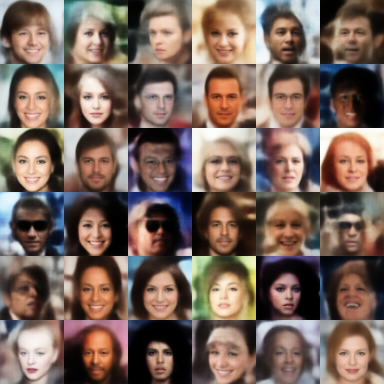}
\end{tabular}

\clearpage
\begin{tabular}{lcccc}
& $\bm{\beta=0.8}$ & $\bm{\beta=1.0}$   \\
\rotatebox[origin=lt]{90}{~~~~~~~~~~~~~~~~~~~~~~~~{\bf $\bm{\beta}$-VAE} Samples}\hspace{-4mm} & \includegraphics[width=0.48\textwidth]{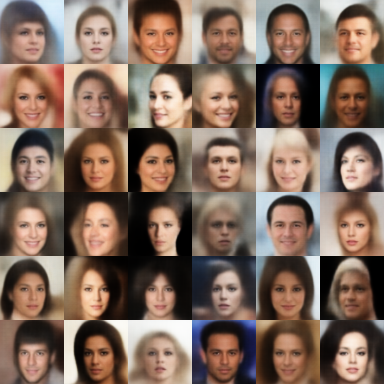}\hspace{-2mm} &
\includegraphics[width=0.48\textwidth]{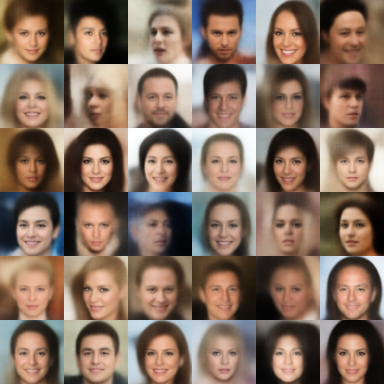}\hspace{-2mm}  \vspace{2mm}\\
\rotatebox[origin=lt]{90}{~~~~~~~~~~~~~~~~~~~~~~~~{\bf GRAE$_{\approx}$} Samples}\hspace{-4mm} & \includegraphics[width=0.48\textwidth]{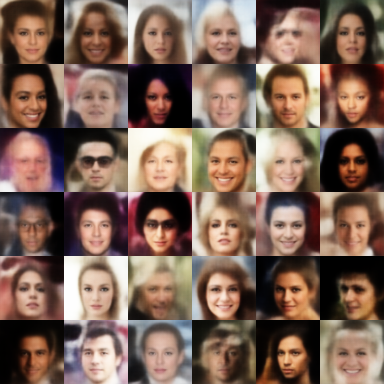}\hspace{-2mm} &
\includegraphics[width=0.48\textwidth]{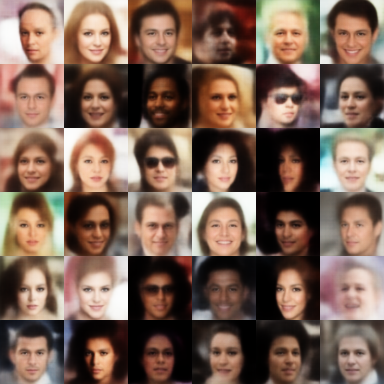}
\end{tabular}

\end{document}